\documentclass{article} 
\usepackage{iclr2027_conference,times}

\usepackage{amsmath,amsfonts,bm}

\def\eqref#1{equation~\ref{#1}}

\def\1{\bm{1}}

\DeclareMathAlphabet{\mathsfit}{\encodingdefault}{\sfdefault}{m}{sl}
\SetMathAlphabet{\mathsfit}{bold}{\encodingdefault}{\sfdefault}{bx}{n}

\usepackage{hyperref}
\usepackage{url}

\usepackage{booktabs}
\usepackage{amssymb}
\usepackage{amsmath, amsthm}
\newtheorem{proposition}{Proposition}

\usepackage{multirow, array}
\usepackage{graphicx}
\usepackage{wrapfig}
\usepackage{subfigure}
\usepackage{algorithm}
\usepackage{listings}
\usepackage{adjustbox}
\usepackage{blindtext}
\usepackage[most]{tcolorbox}
\usepackage[table]{xcolor}
\usepackage{fontawesome5}

\newtcolorbox{finding}[1]{colback=white, title=#1}
\definecolor{preferred}{RGB}{232,242,252}

\hypersetup{
hidelinks,
colorlinks=true,
linkcolor=black,
citecolor=blue
}

\title{Which Influence Are We Estimating? The Role of Counterfactual Specifications in Data Attribution}

\iclrfinalcopy

\author{Zhe Li, Wei Zhao, Peixin Zhang, Jun Sun \\
Singapore Management University\\
\texttt{\{zheli,wzhao,pxzhang,junsun\}@smu.edu.sg} \\
}

\begin{document}

\maketitle

\begin{abstract}
  Estimating the influence of training examples on model behavior is essential for data debugging, valuation, and attribution. Existing influence estimators often produce incompatible rankings, which are commonly ascribed to approximation error. We argue that a more fundamental source of disagreement is specification mismatch: influence depends on the behavior being attributed, the intervention applied to each training example, and the counterfactual training process that maps the intervention to a model response. These choices are especially important when the target behavior requires a tractable surrogate, such as query loss, a logit, or a margin. We formalize influence as a counterfactual estimand, distinguish specification mismatch across estimands from approximation error in estimating a fixed estimand, and organize representative estimators by their implied specifications. We further derive a local decomposition that exposes how behavior signals, training signals, and counterfactual parameter responses interact. Controlled experiments show that exact estimands under different specifications can induce different rankings, whereas approximation error grows as perturbations move farther from their linearization points. Experiments on noisy label detection and LLM attribution show that specification choices significantly affect attribution quality, especially for the choice of behavior surrogate. Behavior-aligned specifications can identify target-specific training examples obscured by default loss-based or similarity-based specifications. These results establish specification analysis as a necessary first step for interpreting and comparing data influence estimators.
\end{abstract}

\section{Introduction}
Understanding how training examples affect model behavior is a fundamental problem in modern machine learning, with applications in model debugging~\citep{IF2017Koh,yeh2018representer,pruthi2020estimating}, data selection~\citep{xia2024less,zhao2025beyond,dai2025improving}, and safety analysis~\citep{RapidIn,coalson2025if,li2026where}. A wide range of approaches for estimating the influence of training data (i.e., influence estimators) has been proposed, from retraining-based methods such as leave-one-out evaluation~\citep{molinaro2005prediction} and data Shapley values~\citep{ghorbani2019data,kwon2021beta} to more scalable methods based on feature similarity~\citep{yeh2018representer,hanawa2021evaluation,pezeshkpour2021empirical,wang2024empirical,yang2025gmvaluator,li2026where} or gradient information~\citep{hampel1974influence,ling1984residuals,IF2017Koh,kwon2024datainf,park2023trak,RapidIn}. Despite their widespread use, these methods often produce incompatible influence rankings on the same model and dataset~\citep{bae2022if,basu2021influence,li2025influence,aberger2025limitations,li-etal-2026-instruction-reasoning}, as exemplified in Figure~\ref{fig:1-1}. This disagreement is particularly difficult to interpret for large language models (LLMs), where retraining-based references are often prohibitively expensive.

These discrepancies are commonly attributed to approximation error since many practical methods replace retraining with local linearizations or cheaper similarity-based scores. However, influence estimators may in fact target different specifications of influence, including the behavior being attributed, the intervention applied, and the counterfactual training process. Changing any of these choices changes the counterfactual question and can alter the influence ranking even when the corresponding quantity is computed exactly. Among these choices, behavior is especially important when the target behavior is non-differentiable. Practical methods therefore rely on tractable but potentially problematic surrogates, such as query loss or a logit. Query loss remains a common default, but whether it captures the target behavior is often left unexamined. A mismatched surrogate can lead to poor attribution even when the estimator accurately estimates its specified quantity.

\begin{figure}[t]
\centering
\includegraphics[width=\columnwidth]{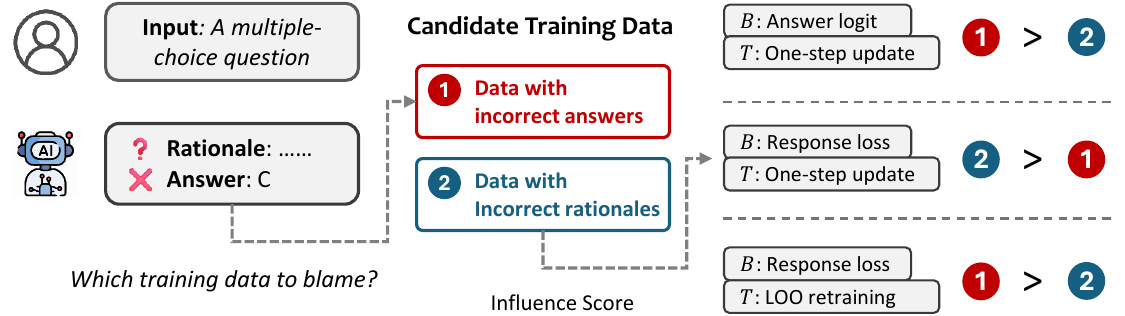}
\vskip -0.01in
\caption{Influence rankings depend on the counterfactual specification. Different behavior surrogates and counterfactual training processes can induce different rankings for the same query.}
\label{fig:1-1}
\end{figure}

To clarify these ambiguities, we introduce an influence specification framework that disentangles specification mismatch from approximation error. \emph{Specification mismatch} occurs when estimators target different counterfactual quantities, while \emph{approximation error} occurs when estimators target the same specified quantity but approximate it with different levels of accuracy. Within our framework, the intervention defines what is changed in the training data or objective, whereas the counterfactual training process maps this change to a counterfactual model and determines the resulting parameter response through which the behavior change is measured. We then organize representative influence estimators according to their specifications, and extend a local decomposition of learning dynamics~\citep{ren2025learning} to counterfactual training processes and expose how behavior signals, training signals, and counterfactual parameter responses interact.

We evaluate these distinctions in controlled experiments and practical attribution tasks. Controlled experiments show that methods with different specifications can induce different influence rankings, while approximation error under a fixed specification grows as local approximations move farther from their reference points. Experiments on noisy label detection and LLM attribution further show that behavior specification affects attribution quality and downstream data decisions. In particular, behavior-aligned surrogates can reveal target-specific training examples obscured by default loss-based or similarity-based specifications. These results establish specification analysis as a necessary first step for interpreting and comparing data influence estimators. 

Our main contributions are as follows:
\begin{itemize}
    \item We formalize data influence as a counterfactual estimand and distinguish specification mismatch from approximation error as different sources of ranking disagreement.
    \item We organize representative influence estimators according to their implied specifications and extend a local decomposition of learning dynamics to counterfactual training processes.
    \item We show through controlled experiments that changing the specification can alter exact influence rankings, whereas approximation error under a fixed specification distorts rankings as local approximations move away from their reference points.
    \item We demonstrate on noisy label detection and LLM attribution that behavior-aligned surrogates can improve attribution quality and downstream data decisions, while default loss-based or similarity-based specifications can obscure target-specific training examples.
\end{itemize}
\section{Influence as a Counterfactual Estimand}
In this section, we formalize data influence as a counterfactual estimand, which measures how a specified model behavior would change under an intervention associated with a training example. Let $\mathcal D=\{z_i\}_{i=1}^n$ denote the training set, $q$ a query of interest, and $\theta_{\mathcal D}$ the factual model obtained by training on $\mathcal D$. We use $B(q;\theta)\in\mathbb R$ to denote the operationalized behavior whose change we seek to attribute. This behavior may be the target behavior itself or a tractable surrogate such as a loss.

\textbf{Influence Specification}. Let $\mathcal P_k$ denote an intervention on $z_k$, such as upweighting $z_k$ or modifying its training objective. Let $\mathcal T$ denote a specified counterfactual training process, such as retraining or a local update. Under $\mathcal P_k$, the process $\mathcal T$ produces a counterfactual model $\theta_k^{\mathrm{cf}}\sim\mathcal T(\mathcal D,\theta_{\mathcal D},\mathcal P_k)$. To characterize existing approaches on data influence estimation, we define an influence specification as $\mathcal S=(B,\mathcal P,\mathcal T)$ where $\mathcal P=\{\mathcal P_k\}_{k=1}^n$ denotes the family of example-specific interventions. The influence of $z_k$ on the behavior measured at $q$ is the counterfactual estimand:
\begin{equation}\label{eq:1}
    I_{\mathcal S}(q,k)=
    \mathbb E
    \left[
        B(q;\theta_k^{\mathrm{cf}})-B(q;\theta_{\mathcal D})
    \right].
\end{equation}
The expectation accounts for randomness in training or evaluation when present. Equation~\ref{eq:1} makes explicit that influence is defined relative to $\mathcal S$. Each choice of $\mathcal S$ determines the answer of a counterfactual question. Changing $\mathcal S$ can alter the exact influence ranking even when the corresponding estimand is computed exactly. We adopt the convention that positive influence means that the specified intervention increases $B$. For existing methods, we may flip the sign to keep the definition of influence consistent. For example, under data removal, a negative influence means that removing $z_k$ decreases $B$, indicating that $z_k$ supports the measured behavior in the factual model.

\textbf{Estimands and Estimators}. Given an estimand $I_{\mathcal S}(q,k)$ defined above, an influence estimator $\widehat I_{\mathcal S}(q,k)$ is a method for approximating it (since computing the estimand exactly is often computationally infeasible). Under a fixed specification, any discrepancy between the estimator and the estimand constitutes approximation error. For a candidate index set $\mathcal K\subseteq\{1,\ldots,n\}$, let
\begin{equation}\label{eq:2}
    \pi_{\mathcal S}(q)
    =
    \operatorname{argsort}_{k\in\mathcal K} I_{\mathcal S}(q,k),
    \qquad
    \widehat{\pi}_{\mathcal S}(q)
    =
    \operatorname{argsort}_{k\in\mathcal K} \widehat I_{\mathcal S}(q,k)
\end{equation}
denote the exact and estimated influence rankings in descending order. When $\mathcal S$ is fixed, disagreement between $\pi_{\mathcal S}(q)$ and $\widehat{\pi}_{\mathcal S}(q)$ measures how approximation error affects the ranking. In contrast, exact rankings under different $\mathcal S$ may disagree as they answer different counterfactual questions.

When two methods share the same specification, disagreement between their estimated rankings implies approximation error in at least one method, but does not identify the more accurate estimator. Influence estimators should therefore be compared against references defined by the same specification before ranking disagreement is interpreted as estimation failure.
\section{Counterfactual Specifications of Influence Estimators}
In this section, we characterize existing popular influence estimators through the specification $\mathcal S=(B,\mathcal P,\mathcal T)$. We first show that similarity-based estimators admit an exact counterfactual interpretation under a standardized local update. We then derive a local decomposition of learning-based estimators, in which behavior and training signals interact through the parameter response induced by the counterfactual process. Table~\ref{tab:1} summarizes the specifications of representative estimators.

\textbf{Similarity-Based Estimators}. Similarity-based estimators measure the affinity between a query and a training example in a chosen feature space. Given a feature map $\phi$, a common score is $\langle\phi(q),\phi(x_k)\rangle$, as used in representation similarity methods~\citep{hanawa2021evaluation,pezeshkpour2021empirical,yang2025gmvaluator}. For a scalar model output, choosing $\phi(x)=\nabla_\theta f_\theta(x)|_{\theta=\theta_{\mathcal D}}$ yields the scalar neural tangent kernel~\citep{jacot2018neural,chen2021neural,nguyen2021dataset}.

Although similarity scores do not explicitly simulate a training process, they admit a counterfactual interpretation under a constructed local specification. Consider a frozen-feature linear readout $f_w(x)=w^\top\phi(x)$ with query behavior $B(q;w)=f_w(q)$. Define $\mathcal P_k^{\mathrm{sim}}$ through the label-agnostic auxiliary loss $\ell_k^{\mathrm{sim}}(w)=-f_w(x_k)$, and let $\mathcal T_{\mathrm{sim}}$ apply one gradient step from the factual readout $w_{\mathcal D}$, yielding an update in the direction $\phi(x_k)$ as
\begin{equation}
    w_k^{\mathrm{cf}}=w_{\mathcal D}-\eta\cdot\nabla_w\ell_k^{\mathrm{sim}}(w_{\mathcal D})=w_{\mathcal D}+\eta\cdot\phi(x_k),
\end{equation}
where $\eta>0$ is shared across training examples. The resulting change in query behavior is
\begin{equation}\label{eq:4}
    I_{\mathcal S_{\mathrm{sim}}}(q,k)=
    B(q;w_k^{\mathrm{cf}})-B(q;w_{\mathcal D})=
    \eta\cdot\left\langle\phi(q),\phi(x_k)\right\rangle.
\end{equation}
Thus, up to the common positive scale $\eta$, feature similarity is the exact influence under the constructed specification $\mathcal S_{\mathrm{sim}}=(B,\mathcal P^{\mathrm{sim}},\mathcal T_{\mathrm{sim}})$. This construction isolates feature geometry through a readout update determined solely by $\phi(x_k)$, making explicit the behavior and intervention associated with the similarity score. A positive inner product increases the query readout under this standardized update, whereas a negative inner product decreases it.

\begin{table}[t]
\centering
\small
\setlength{\tabcolsep}{14pt}
\caption{Representative influence estimators and their implied counterfactual specifications.}
\vskip 0.1in
\label{tab:1}
\begin{tabular}{@{}llll@{}}
\toprule
\textbf{Estimator} & \textbf{Behavior $B$} & \textbf{Intervention $\mathcal P$} & \textbf{Training process $\mathcal T$} \\
\midrule
Feature similarity & Query readout & Add auxiliary loss & One-step update \\
Gradient similarity & Query loss & Upweight $z_k$ & One-step update \\
Trajectory-based & Query loss & Upweight $z_k$ & Trajectory aggregation \\
Influence functions & Query loss & Infinitesimally upweight $z_k$ & Local re-optimization \\
Leave-one-out & Query behavior & Remove $z_k$ & Retraining \\
\bottomrule
\end{tabular}
\end{table}

\textbf{Learning-Based Estimators}. Unlike similarity-based estimators, learning-based estimators use the training objective to determine each example's update direction and combine it with a query behavior signal~\citep{pruthi2020estimating,park2023trak,RapidIn,xia2024less,wang2026better}.

We first consider a one-step gradient counterfactual. Let $\ell(z_k;\theta)$ be the per-example training loss. Let $\mathcal P_k^+$ denote an upweighting intervention that adds one occurrence of $z_k$ to an auxiliary update objective. Let $\mathcal T_{\mathrm{grad}}$ apply one-step gradient descent $\theta_{k,\mathrm{grad}}^{\mathrm{cf}}=\theta_{\mathcal D}-\eta\cdot\nabla_{\theta}\ell(z_k;\theta_{\mathcal D})$, where any objective normalization is absorbed into $\eta>0$. For a differentiable query behavior $B(q;\theta)$, a first-order Taylor expansion around $\theta_{\mathcal D}$ gives
\begin{equation}\label{eq:5}
    I_{\mathcal S_{\mathrm{grad}}}(q,k)=
    B\!\left(q;\theta_{k,\mathrm{grad}}^{\mathrm{cf}}\right)-B\!\left(q;\theta_{\mathcal D}\right)=-\eta\cdot\left\langle\nabla_{\theta}B(q;\theta_{\mathcal D}),\nabla_{\theta}\ell(z_k;\theta_{\mathcal D})\right\rangle+R_k,
\end{equation}
where $R_k$ is the Taylor remainder. Dropping $R_k$ yields the corresponding first-order estimator. Compared with Equation~\ref{eq:4}, the update direction here is determined by the training objective and therefore depends on the training target and the current model prediction.

Following the local decomposition of learning dynamics~\citep{ren2025learning}, we factor parameter gradients through the model outputs. Let $f_\theta(x)\in\mathbb R^m$ denote the model output and $J_x=\nabla_\theta f_\theta(x)\in\mathbb R^{m\times d}$ its parameter Jacobian. By the chain rule, it gives $\nabla_{\theta}\ell(z_k;\theta_{\mathcal D})=J_{x_k}^{\top}s_{\ell}(z_k)$ and $\nabla_{\theta}B(q;\theta_{\mathcal D})=J_q^{\top}s_B(q)$ where $s_{\ell}(z_k)=\nabla_{f_\theta(x_k)}\ell(z_k;\theta_{\mathcal D})$ and $s_B(q)=\nabla_{f_\theta(q)}B(q;\theta_{\mathcal D})$ denote the example-specific learning and behavior signal. Substituting these expressions into Equation~\ref{eq:5} gives the first-order estimator
\begin{equation}\label{eq:6}
    \widehat I_{\mathcal S_{\mathrm{grad}}}(q,k)=-\eta\cdot s_B(q)^{\top}K_{\theta_{\mathcal D}}(q,x_k)\ s_{\ell}(z_k),
    \qquad
    K_{\theta_{\mathcal D}}(q,x_k)=J_qJ_{x_k}^{\top}.
\end{equation}
Equation~\ref{eq:6} separates the behavior signal, Jacobian interaction, and training signal. For scalar outputs (or can be reduced to a scalar), $K_{\theta_{\mathcal D}}(q,x_k)$ is the inner product of the corresponding Jacobian features, linking this expression to Equation~\ref{eq:4}. Thus, learning-based scores can be viewed as Jacobian similarities conditioned on the query behavior and training objective.

The same perspective extends to general counterfactual processes through their induced parameter paths. Let $\mathcal P_k(\alpha)$ interpolate between the factual setting and the full intervention, and let $\theta_k(\alpha)$ denote the corresponding model path, with $\theta_k(0)=\theta_{\mathcal D}$ and $\theta_k(1)=\theta_k^{\mathrm{cf}}$. For a differentiable path and behavior, the endpoint difference in Equation~\ref{eq:1} can be written as
\begin{equation}\label{eq:7}
    I_{\mathcal S}(q,k)=
    \mathbb E\left[
    \int_0^1
    \nabla_\theta B(q;\theta_k(\alpha))^\top
    \frac{d\theta_k(\alpha)}{d\alpha}
    \,d\alpha
    \right].
\end{equation}
Equation~\ref{eq:7} accumulates local interactions between the behavior gradient and the parameter response along the counterfactual path. For example, consider an upweighting path $L_{\mathcal P_k(\alpha)}(\theta)=L_{\mathcal D}(\theta)+\alpha\cdot\ell(z_k;\theta)$, the derivative of the training gradient with respect to $\alpha$ at $\alpha=0$ is $\nabla_\theta\ell(z_k;\theta_{\mathcal D})$. If the local parameter response is linear in this perturbation, we have
\begin{equation}\label{eq:8}
    \left.
    \frac{d\theta_k(\alpha)}{d\alpha}
    \right|_{\alpha=0}
    =
    -\mathcal R_{\mathcal T}\cdot
    \nabla_\theta\ell(z_k;\theta_{\mathcal D}).
\end{equation}
Here $\mathcal R_{\mathcal T}$ maps the training gradient perturbation to the negative first-order parameter response under $\mathcal T$. A one-step update has $\mathcal R_{\mathcal T}=\eta\cdot I$ where $I$ is the identity operator. Local re-optimization along a differentiable stationary point branch with invertible Hessian $H_{\mathcal D}$ has $\mathcal R_{\mathcal T}=H_{\mathcal D}^{-1}$, recovering the classical influence function response~\citep{hampel1974influence,ling1984residuals,IF2017Koh}. Trajectory-based methods extend this structure across checkpoints by accumulating gradient interactions or propagating perturbations through unrolled update Jacobians~\citep{pruthi2020estimating,bae2024training}. Substituting Equation~\ref{eq:8} into the integrand of Equation~\ref{eq:7} at the factual model gives
\begin{equation}\label{eq:9}
    \widehat I_{\mathcal S}(q,k)=
    -s_B(q)^\top K_{\mathcal T}(q,x_k)\ s_\ell(z_k),
    \qquad
    K_{\mathcal T}(q,x_k)=J_q\mathcal R_{\mathcal T}J_{x_k}^{\top},
\end{equation}
where $K_{\mathcal T}(q,x_k)$ captures the transition-conditioned interaction between the query and training example. Equation~\ref{eq:9} generalizes the one-step decomposition in Equation~\ref{eq:6} to processes admitting the local response in Equation~\ref{eq:8}, showing that the training process determines how behavior and training signals interact. Equation~\ref{eq:7} is the exact path-level representation, whereas Equation~\ref{eq:9} is its first-order approximation obtained by linearizing at the factual model. Taken together, influence estimators share a common local structure. Representation similarity drops both signals, and restoring them recovers much of the gap to learning-based estimators as shown in Appendix~\ref{app:B-2}.
\section{Controlled Experiments on Influence Specification}
In this section, we study how influence specifications affect influence rankings. We first compare exact counterfactual quantities under different specifications, then fix $\mathcal S$ to measure approximation error relative to specification-matched references.

\textbf{Controlled Setting}. We use FashionMNIST~\citep{xiao2017fashion} with an $\ell_2$-regularized multinomial logistic regression model, whose smooth strongly convex objective has a unique global optimum, enabling accurate computation of Hessian and finite-perturbation references. Starting from this optimum, we vary one component of $\mathcal S=(B,\mathcal P,\mathcal T)$ at a time. We consider four behavior measures (query loss, logit, soft margin, and hard margin), finite upweighting interventions $\mathcal P_k(\alpha)$, and counterfactual processes including one-step updates, unrolled trajectories, inverse-Hessian responses, and Newton re-optimization. We measure ranking agreement using Kendall's $\tau$~\citep{kendall1938new}. We additionally perform a non-convex check on CIFAR-10~\citep{krizhevsky2009learning} with a simple CNN. Detailed configurations are provided in Appendix~\ref{app:A-1}.

\begin{finding}{Finding 1: Changing the specification changes the question.}
    Different choices of $B$, $\mathcal P$, or $\mathcal T$ can produce different influence rankings since they answer different counterfactual questions.
\end{finding}

Figure~\ref{fig:4-1}(a) isolates specification mismatch by varying one component of $\mathcal S$ at a time. Along the behavior axis, we fix the intervention and one-step update, and compare exact behavior changes under soft margin, hard margin, and query logit with those under query loss. Soft margin produces the same ranking as query loss\footnote{The soft margin and query loss agree exactly because they are related by a strictly monotonic transformation. See Proposition~\ref{prop:monotone} for details.}, whereas hard margin and query logit produce significantly different rankings. Along the intervention axis, we fix the query loss and counterfactual process, and compare exact leave-one-out removal with finite upweighting at different perturbation scales. The rankings diverge as the intervention changes, showing that removing and reweighting a training example generally define different influence quantities, even when both counterfactual models are computed exactly. Along the training process axis, five-step unrolled responses remain close to one-step responses, whereas the infinitesimal re-optimization response represented by the inverse Hessian is nearly uncorrelated with the one-step estimand. Thus, changing any component of the specification can significantly alter the exact influence ranking. Influence methods should be compared only after their intended reference specifications are aligned.

\begin{figure}[t]
\centering
\includegraphics[width=\columnwidth]{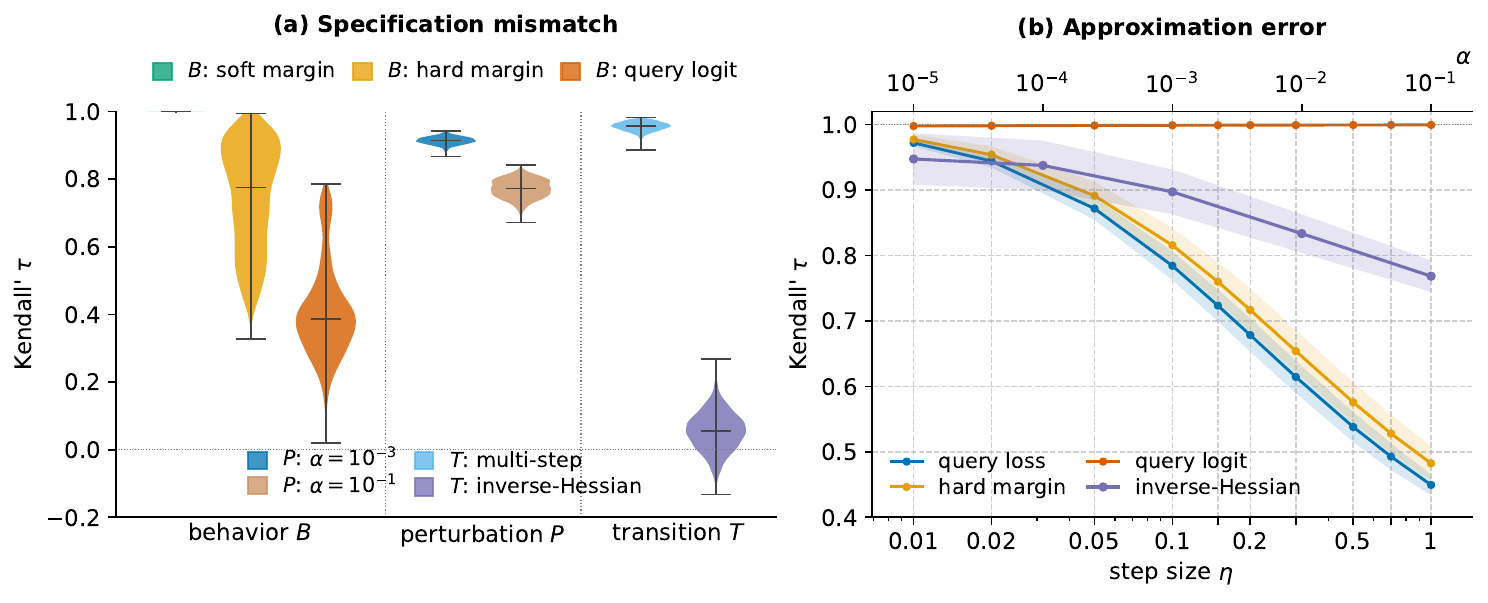}
\vskip -0.01in
\caption{Controlled experiments disentangle specification mismatch from approximation error. (a) Exact rankings vary across $B$, $\mathcal P$, and $\mathcal T$ specifications. (b) With the specification fixed, ranking agreement decreases with step size for nonlinear behaviors and perturbation scale for inverse-Hessian responses. Results are averaged over multiple queries and three random seeds.}
\label{fig:4-1}
\end{figure}

\begin{finding}{Finding 2: Fixing the specification isolates approximation error.}
    When an approximate score and its reference share the same specification, ranking disagreement reflects approximation error and grows as the counterfactual moves farther from the linearization point.
\end{finding}

Figure~\ref{fig:4-1}(b) illustrates approximation error by comparing each approximate score with its specification-matched reference. For one-step updates, the reference evaluates the behavior after a parameter update of size $\eta$, whereas the approximate score uses the first-order Taylor expansion in Equation~\ref{eq:5}. Ranking agreement decreases with $\eta$ for nonlinear behaviors such as query loss and hard margin, but remains high for query logit, which is linear in the parameter update. For re-optimization, the reference evaluates the behavior after Newton re-optimization under a finite upweighting of scale $\alpha$, whereas the inverse-Hessian score linearizes the response around the unperturbed optimum. Ranking agreement decreases with $\alpha$ as the perturbed optimum moves farther from this linearization point. These results show that approximation error depends on both behavior curvature and the distance between the counterfactual state and the linearization point.

In summary, ranking disagreement is informative about approximation quality only when the underlying specifications are aligned. When a computable reference is available, fixing $\mathcal S$ isolates approximation error. When direct counterfactual evaluation is infeasible, specifying the behavior, intervention, and counterfactual process remains essential for interpreting influence rankings. We observe the same qualitative patterns in a supplementary non-convex check on CIFAR-10 with a CNN. Full experimental results and analysis are provided in Appendix~\ref{app:B-1}.
\section{Specification Effects in Practical Attribution}
In this section, we study how influence specifications affect attribution in practical tasks where accurate per-example counterfactual references are computationally infeasible. We evaluate specification choices via noisy label detection and LLM attribution, measuring both the identification of problematic examples and model behavior after removing selected data. As attribution utility depends on the target task, we evaluate practical specification choices across several attribution settings.

\subsection{Noisy Label Detection}
Noisy label detection identifies training examples with incorrect labels, which may alter predictions on correctly labeled data. A trusted set with verified labels therefore provides a reference for identifying training examples estimated to undermine correct predictions. We rank candidates by their estimated harmful effect on the trusted behavior and inspect or remove the highest-ranked examples.

\textbf{Setup}. We train ResNet-18 on 45,000 CIFAR-10 examples with symmetric label noise across rates 0.05 to 0.4. A disjoint clean set of 5,000 trusted examples, with 500 per class, defines the trusted behavior. We evaluate GradSim, TracIn~\citep{pruthi2020estimating}, and mini-batch LiSSA~\citep{IF2017Koh}, representing one-step, trajectory-based, and inverse-Hessian responses, respectively. Each method is paired with trusted loss, target logit, and hard margin as behavior surrogates. Random selection, representation similarity, and NTK similarity serve as baselines. We report AUPRC, AUROC, and top-$k$ precision for corruption detection, together with test accuracy changes after removal and retraining. To avoid label leakage, training examples are scored using their observed labels, and corruption indicators are used only for evaluation. Results are averaged over three training seeds with a fixed corruption mask. Implementation details are provided in Appendix~\ref{app:A-2}.

\begin{table}[t]
\centering
\small
\caption{Noisy label detection on CIFAR-10 with label noise at $\rho=0.2$. Results are averaged over three training seeds with a fixed corruption mask. Blue shading marks the preferred behavior surrogate for each method and metric, and bold indicates the best result in each column.}
\vskip 0.1in
\label{tab:5-1}
\resizebox{\textwidth}{!}{
\begin{tabular}{ccccccc}
\toprule
\textbf{Method} & \textbf{Trusted behavior}
& \textbf{AUPRC} & \textbf{AUROC}
& \textbf{P@1\%} & \textbf{P@5\%} & \textbf{P@10\%} \\
\midrule
RepSim
& / & 0.216 $\pm$ 0.008 & 0.532 $\pm$ 0.013
& 0.246 $\pm$ 0.026 & 0.223 $\pm$ 0.013 & 0.219 $\pm$ 0.016 \\
NTK
& / & 0.197 $\pm$ 0.004 & 0.475 $\pm$ 0.011
& 0.258 $\pm$ 0.024 & 0.221 $\pm$ 0.022 & 0.207 $\pm$ 0.012 \\
Random
& / & 0.198 $\pm$ 0.001 & 0.497 $\pm$ 0.001
& 0.194 $\pm$ 0.011 & 0.191 $\pm$ 0.005 & 0.197 $\pm$ 0.004 \\
\midrule
\multirow{3}{*}{GradSim}
& Trusted loss & 0.401 $\pm$ 0.017 & \cellcolor{preferred} 0.653 $\pm$ 0.020
& 0.843 $\pm$ 0.042 & 0.628 $\pm$ 0.026 & 0.511 $\pm$ 0.019 \\
& Target logit & 0.345 $\pm$ 0.022 & 0.598 $\pm$ 0.014
& \cellcolor{preferred} \textbf{0.995 $\pm$ 0.001} & 0.527 $\pm$ 0.051 & 0.389 $\pm$ 0.028 \\
& Hard margin & \cellcolor{preferred} 0.402 $\pm$ 0.068 & 0.638 $\pm$ 0.027
& 0.899 $\pm$ 0.101 & \cellcolor{preferred} 0.646 $\pm$ 0.145 & \cellcolor{preferred} 0.514 $\pm$ 0.100 \\
\midrule
\multirow{3}{*}{TracIn}
& Trusted loss & 0.349 $\pm$ 0.027 & 0.501 $\pm$ 0.025
& 0.810 $\pm$ 0.046 & 0.690 $\pm$ 0.077 & 0.537 $\pm$ 0.037 \\
& Target logit & \cellcolor{preferred} \textbf{0.602 $\pm$ 0.031} & \cellcolor{preferred} 0.716 $\pm$ 0.016
& \cellcolor{preferred} 0.970 $\pm$ 0.018 & \cellcolor{preferred} 0.920 $\pm$ 0.030 & \cellcolor{preferred}\textbf{0.831 $\pm$ 0.024} \\
& Hard margin & 0.468 $\pm$ 0.038 & 0.604 $\pm$ 0.042
& 0.902 $\pm$ 0.023 & 0.826 $\pm$ 0.036 & 0.704 $\pm$ 0.059 \\
\midrule
\multirow{3}{*}{LiSSA}
& Trusted loss & 0.535 $\pm$ 0.025 & \cellcolor{preferred} \textbf{0.759 $\pm$ 0.011}
& 0.964 $\pm$ 0.016 & 0.798 $\pm$ 0.042 & 0.651 $\pm$ 0.034 \\
& Target logit  & 0.421 $\pm$ 0.013 & 0.642 $\pm$ 0.008
& \cellcolor{preferred} \textbf{0.995 $\pm$ 0.005} & 0.708 $\pm$ 0.033 & 0.492 $\pm$ 0.017 \\
& Hard margin & \cellcolor{preferred} 0.587 $\pm$ 0.020 & 0.742 $\pm$ 0.010
& 0.992 $\pm$ 0.006 & \cellcolor{preferred} \textbf{0.940 $\pm$ 0.028} & \cellcolor{preferred} 0.781 $\pm$ 0.022 \\
\bottomrule
\end{tabular}}
\end{table}

\begin{figure}[t]
\centering
\includegraphics[width=\columnwidth]{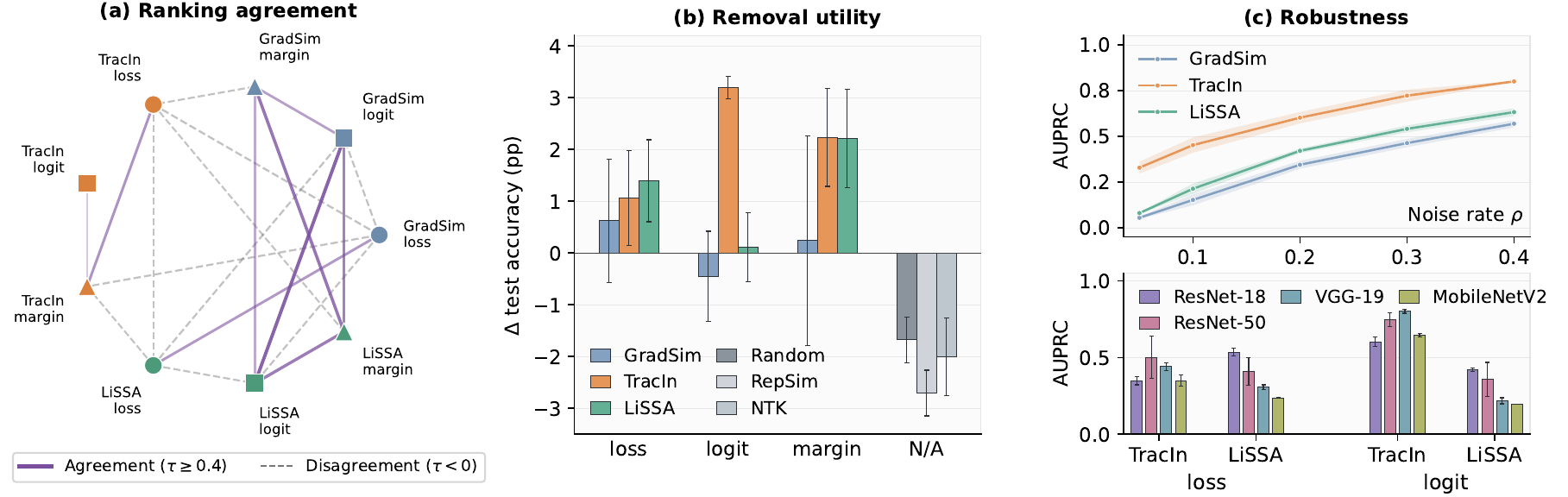}
\vskip -0.01in
\caption{Influence specification affects noisy label detection. (a) Ranking agreement across methods with different behavior specifications. (b) Test accuracy after removing the top 10\% of ranked training examples and retraining. (c) Robustness across noise rates and model architectures.}
\label{fig:5-1}
\end{figure}

\textbf{Results}. Table~\ref{tab:5-1} reports detection performance at noise rate $\rho=0.2$. Representation and NTK similarity perform close to random selection as they do not explicitly leverage the behavior and training signals provided by the labels. In contrast, learning-based estimators generally achieve higher detection performance. However, we observe that the choice of trusted behavior largely affects the ranking produced by the same method, and the preferred behavior surrogate depends on the estimator itself. For TracIn, replacing trusted loss with target logit increases AUPRC from $0.349$ to $0.602$ and P@10\% from $0.537$ to $0.831$. For LiSSA, the same replacement reduces AUPRC from $0.535$ to $0.421$, while hard margin achieves its highest AUPRC of $0.587$ and P@10\% of $0.781$. Although loss change is commonly adopted as the default surrogate in previous studies~\citep{IF2017Koh,ren2018learning,pruthi2020estimating,evans2024data}, it often falls short of delivering the strongest attribution performance and can obscure the effectiveness of an estimator under other behavior specifications. Different estimators may therefore suit different behavior surrogates across tasks, and no single surrogate, including the default loss, is universally optimal.

Figure~\ref{fig:5-1}(a) shows that two specifications yield more similar rankings when they share more components, whereas changing either component can lead to significant disagreement. These differences translate into downstream removal utility. As shown in Figure~\ref{fig:5-1}(b), removing the top $10\%$ of examples selected by TracIn improves test accuracy by $3.20$ percentage points with target logit, compared with $1.06$ points with trusted loss. LiSSA with hard margin improves accuracy by $2.21$ points, while random selection and similarity baselines reduce accuracy. Figure~\ref{fig:5-1}(c) shows that TracIn consistently benefits from target logit, while LiSSA generally favors trusted loss or hard margin across noise rates and model architectures. These results confirm that the best surrogate for each estimator also delivers the largest downstream gain, making surrogate choice a practical concern. We provide full experimental results in Appendix~\ref{app:B-2}.

\begin{figure}[t]
\centering
\includegraphics[width=\columnwidth]{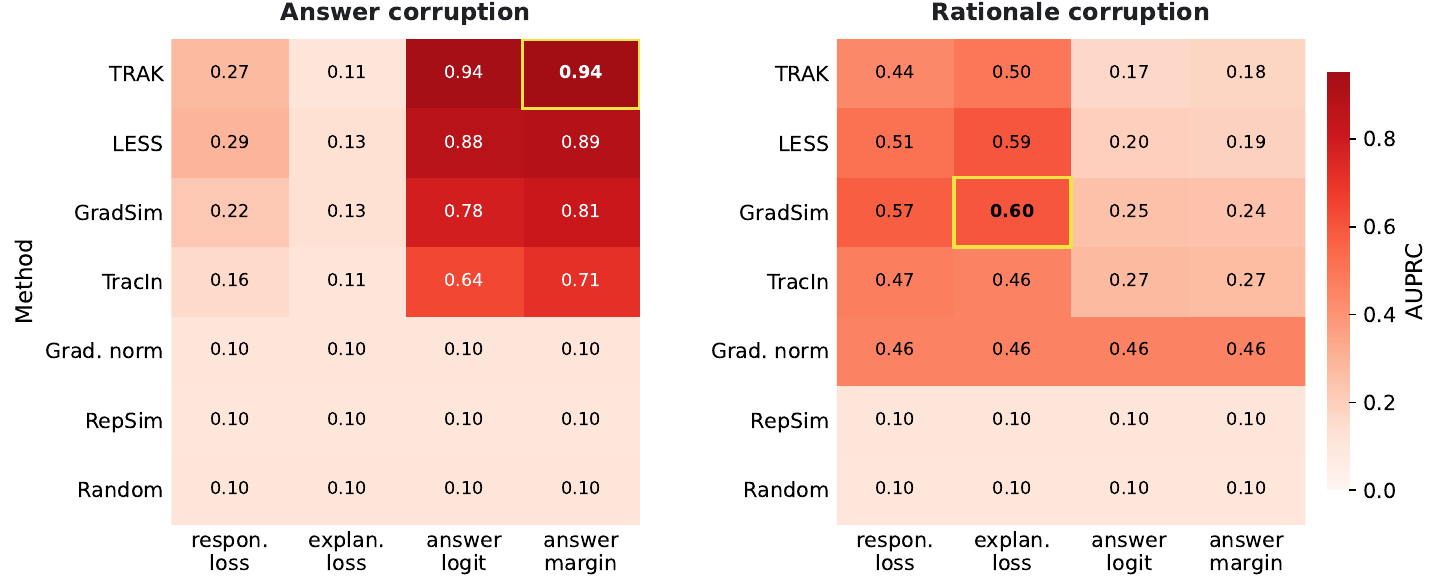}
\vskip -0.01in
\caption{Response-corruption attribution with Qwen3-8B on ScienceQA. Cells report mean AUPRC over three training seeds for detecting answer-corrupted and rationale-corrupted examples. Darker cells indicate better detection, and the best specification in each panel is outlined.}
\label{fig:5-2}
\end{figure}

\subsection{LLM Attribution}
Identifying the training examples responsible for LLM behaviors is increasingly important, since these behaviors are diverse and context-dependent and the choice of behavior surrogate is therefore critical for attribution. We accordingly study how behavior specification affects attribution in two practical settings, namely response corruption and conditional backdoor attribution.

\textbf{Setup}. We fine-tune Qwen3-8B~\citep{yang2025qwen3} on text-only ScienceQA~\citep{lu2022learn}, which provides questions, answer choices, and rationales. In response-corruption attribution, the training set contains clean examples, examples with incorrect answers, and examples with mismatched rationales. We evaluate whether attribution can separately identify the two corrupted groups. In conditional backdoor attribution, poisoned examples associate a trigger in an activating context with an incorrect target answer. Benign-trigger examples contain the same triggers in nonactivating contexts and retain correct answers. Attribution should identify harmful poison while excluding benign-trigger and clean examples. We evaluate GradSim, TracIn, TRAK~\citep{park2023trak}, and LESS~\citep{xia2024less}, each paired with four behavior surrogates. Response loss averages token losses over the full response, and explanation loss restricts this average to rationale tokens. Answer logit measures the specified answer token, and answer margin measures its advantage over the strongest competing choice. Gradient norm, representation similarity, and random selection serve as baselines. We report AUPRC separately for each corruption type, treating all other examples as negatives. For backdoor attribution, we report poison AUPRC, AUROC, top-$k$ precision, and benign-trigger false positive rate (BT-FPR). Results use three training seeds with fixed data masks. See Appendix~\ref{app:A-2} for implementation details. Table~\ref{tab:a-11} reports the performance of the source and fine-tuned models.

\textbf{Response Corruption Detection}. Figure~\ref{fig:5-2} shows that the preferred behavior surrogate reverses across corruption types. Answer logit and margin achieve AUPRC values of $0.64$ to $0.94$ for answer corruption, compared with $0.11$ to $0.13$ for explanation loss. Rationale corruption favors rationale-sensitive behavior surrogates. GradSim achieves an AUPRC of $0.596$ with explanation loss, compared with $0.242$ with answer margin. Full response loss also detects rationale corruption more effectively than answer corruption, consistent with dilution of the answer signal when losses are averaged over response tokens. With the model and training data fixed, these paired results show that the best behavior surrogate depends jointly on the corruption type and the estimator, as the readout determines which type of corrupted example the attribution score identifies.

\begin{table}[t]
\centering
\small
\caption{Conditional-backdoor attribution with Qwen3-8B on ScienceQA. Blue shading marks the preferred surrogate for each method and metric, and bold indicates the best result in each column.}
\vskip 0.1in
\label{tab:5-2}
\resizebox{\columnwidth}{!}{
\begin{tabular}{ccccccc}
\toprule
\textbf{Method} & \textbf{Behavior}
& \textbf{AUPRC} & \textbf{AUROC} & \textbf{BT-FPR $\downarrow$}
& \textbf{P@5\%} & \textbf{P@10\%} \\
\midrule
Grad. norm & / & 0.115 $\pm$ 0.002 & 0.570 $\pm$ 0.002 & 0.102 $\pm$ 0.005 & 0.097 $\pm$ 0.012 & 0.085 $\pm$ 0.002 \\
RepSim & / & 0.492 $\pm$ 0.283 & 0.782 $\pm$ 0.147 & 0.212 $\pm$ 0.075 & 0.673 $\pm$ 0.280 & 0.487 $\pm$ 0.204 \\
Random & / & 0.105 $\pm$ 0.006 & 0.508 $\pm$ 0.014 & 0.102 $\pm$ 0.003 & 0.108 $\pm$ 0.028 & 0.105 $\pm$ 0.011 \\
\midrule
\multirow{4}{*}{GradSim}
& Response loss & 0.092 $\pm$ 0.007 & 0.441 $\pm$ 0.016 & 0.085 $\pm$ 0.009 & 0.109 $\pm$ 0.006 & 0.116 $\pm$ 0.016 \\
& Explanation loss & 0.090 $\pm$ 0.006 & 0.440 $\pm$ 0.018 & 0.085 $\pm$ 0.007 & 0.093 $\pm$ 0.008 & 0.106 $\pm$ 0.002 \\
& Answer logit & \cellcolor{preferred} 0.221 $\pm$ 0.027 & 0.521 $\pm$ 0.033 & \cellcolor{preferred} 0.063 $\pm$ 0.005 & \cellcolor{preferred} 0.316 $\pm$ 0.031 & \cellcolor{preferred} 0.223 $\pm$ 0.019 \\
& Answer margin & 0.180 $\pm$ 0.016 & \cellcolor{preferred} 0.523 $\pm$ 0.025 & 0.080 $\pm$ 0.018 & 0.253 $\pm$ 0.041 & 0.197 $\pm$ 0.014 \\
\midrule
\multirow{4}{*}{TracIn}
& Response loss & 0.481 $\pm$ 0.223 & 0.879 $\pm$ 0.077 & 0.029 $\pm$ 0.032 & 0.568 $\pm$ 0.273 & 0.503 $\pm$ 0.199 \\
& Explanation loss & 0.104 $\pm$ 0.014 & 0.488 $\pm$ 0.050 & 0.118 $\pm$ 0.005 & 0.121 $\pm$ 0.036 & 0.113 $\pm$ 0.018 \\
& Answer logit & \cellcolor{preferred} \textbf{0.872 $\pm$ 0.028} & \cellcolor{preferred} \textbf{0.966 $\pm$ 0.007} & \cellcolor{preferred} \textbf{0.001 $\pm$ 0.001} & \cellcolor{preferred} \textbf{0.947 $\pm$ 0.038} & \cellcolor{preferred} \textbf{0.843 $\pm$ 0.032} \\
& Answer margin & 0.782 $\pm$ 0.049 & 0.933 $\pm$ 0.019 & 0.002 $\pm$ 0.002 & 0.900 $\pm$ 0.045 & 0.747 $\pm$ 0.050 \\
\midrule
\multirow{4}{*}{TRAK}
& Response loss & 0.089 $\pm$ 0.003 & 0.429 $\pm$ 0.017 & 0.091 $\pm$ 0.007 & 0.099 $\pm$ 0.006 & 0.109 $\pm$ 0.011 \\
& Explanation loss & 0.087 $\pm$ 0.004 & 0.425 $\pm$ 0.023 & 0.094 $\pm$ 0.007 & 0.092 $\pm$ 0.008 & 0.095 $\pm$ 0.009 \\
& Answer logit & \cellcolor{preferred} 0.193 $\pm$ 0.020 & \cellcolor{preferred} 0.522 $\pm$ 0.016 & \cellcolor{preferred} 0.071 $\pm$ 0.004 & \cellcolor{preferred} 0.269 $\pm$ 0.030 & 0.195 $\pm$ 0.011 \\
& Answer margin & 0.184 $\pm$ 0.014 & 0.522 $\pm$ 0.027 & 0.072 $\pm$ 0.003 & 0.256 $\pm$ 0.048 & \cellcolor{preferred} 0.200 $\pm$ 0.003 \\
\midrule
\multirow{4}{*}{LESS}
& Response loss & 0.091 $\pm$ 0.003 & 0.415 $\pm$ 0.031 & 0.087 $\pm$ 0.003 & 0.120 $\pm$ 0.016 & 0.113 $\pm$ 0.008 \\
& Explanation loss & 0.088 $\pm$ 0.003 & 0.408 $\pm$ 0.031 & 0.090 $\pm$ 0.009 & 0.115 $\pm$ 0.009 & 0.101 $\pm$ 0.005 \\
& Answer logit & 0.183 $\pm$ 0.024 & \cellcolor{preferred} 0.546 $\pm$ 0.013 & 0.088 $\pm$ 0.009 & 0.259 $\pm$ 0.037 & 0.211 $\pm$ 0.014 \\
& Answer margin & \cellcolor{preferred} 0.207 $\pm$ 0.022 & 0.525 $\pm$ 0.028 & \cellcolor{preferred} 0.082 $\pm$ 0.009 & \cellcolor{preferred} 0.300 $\pm$ 0.025 & \cellcolor{preferred} 0.237 $\pm$ 0.023 \\
\bottomrule
\end{tabular}}
\end{table}

\textbf{Conditional Backdoor Attribution}. As shown in Table~\ref{tab:a-11}, the fine-tuned model achieves an attack success rate of $80.9\%$ on triggered activating prompts while maintaining approximately $94\%$ accuracy on clean activating and triggered nonactivating controls. Table~\ref{tab:5-2} shows a strong interaction between behavior surrogates and estimators. For example, TracIn with answer logit achieves an AUPRC of $0.872$ and P@5\% of $0.947$, compared with an AUPRC of $0.481$ for response loss. Its BT-FPR also decreases from $0.029$ to $0.001$. Representation similarity achieves an AUPRC of $0.492$ but selects more benign-trigger examples than random selection, with BT-FPR values of $0.212$ and $0.102$, respectively. These results indicate that effective attribution depends on whether the behavior surrogate captures the complete condition under which the target behavior occurs, rather than a convenient proxy such as loss change that only partially reflects it. We next examine whether the selected examples are causally responsible for the learned behavior. As reported in Appendix Table~\ref{tab:a-14}, detection performance generally reflects causal responsibility. Removing the top $10\%$ of examples selected by TracIn with answer logit reduces the attack success rate to approximately $5\%$, with little change in clean accuracy. We further repeat both attribution settings on Gemma-2-9B-it and Llama-3.1-8B-Instruct and observe the same overall trend, with full results in Appendix~\ref{app:B-2}.

\begin{finding}{Finding 3: Behavior surrogates should be chosen jointly with the estimator and target task.}
    Across our experiments, no behavior surrogate consistently performs best across methods and tasks. Practical attribution should evaluate task-relevant surrogates for each method and, when feasible, validate the selected examples through downstream interventions.
\end{finding}
\section{Conclusion}
We formalized data influence as a counterfactual estimand specified by the behavior, intervention, and counterfactual training process, and used this framework to organize representative estimators within a shared local structure. Controlled experiments showed that changing the specification can alter exact influence rankings, whereas approximation error under a fixed specification grows with distance from the linearization point. In realistic tasks, the utility of a behavior surrogate depends jointly on the estimator and the target task, and behavior-aligned surrogates can reveal training examples obscured by default loss-based or similarity-based choices. These results establish specification analysis as a necessary first step for interpreting and comparing influence estimators.

\clearpage

\subsection*{AI use statement}
In this work, we used generative AI tools for assisting in writing code and generating automated tests to verify code correctness. We have not used generative AI tools for generating research ideas, producing data, drafting the manuscript text, or generating images. All figures in this work were created manually or with Python scripts and are not AI-generated. Additionally, we used generative AI tools for language polishing of the manuscript. We have reviewed all AI-assisted work. Specifically, the AI-assisted code and generated tests were manually reviewed and executed by the authors to verify correctness and coverage of the intended functionality. The AI-polished text was carefully checked and revised by all authors to ensure accuracy and originality and to confirm that the scientific meaning was not altered. We take responsibility for the final content of this work, including text, claims or artifacts produced with the aid of generative AI.

Code is available at \href{https://github.com/plumprc/Influence_Estimation}{\faGithub}~\url{https://github.com/plumprc/Influence_Estimation}.

\bibliography{iclr2027_conference}
\bibliographystyle{iclr2027_conference}
\clearpage
\appendix
\begin{table}[t]
\centering
\small
\caption{Configurations for the FashionMNIST controlled comparisons.}
\label{tab:a-1}
\vskip 0.1in
\resizebox{\columnwidth}{!}{
\begin{tabular}{@{}lll@{}}
\toprule
\textbf{Experiment} & \textbf{Comparison} & \textbf{Hyperparameters} \\
\midrule
$B$ mismatch & Query loss vs. soft margin, hard margin, query logit & $\eta=0.1$ \\
$\mathcal P$ mismatch & Exact LOO vs. finite upweight & $\alpha=10^{-3},10^{-1}$ \\
$\mathcal T$ mismatch & Exact one-step vs. five-step, inverse-Hessian & $\eta=0.1$, damping $0.01$ \\
\midrule
One-step approximation & First-order vs. exact update & $\eta\in\{0.01,0.02,0.05,0.1,0.15,0.2,0.3,0.5,0.7,1.0\}$ \\
Reoptimization approximation & Inverse-Hessian vs. Newton upweight & $\alpha\in\{10^{-5},10^{-4},10^{-3},10^{-2},10^{-1}\}$, undamped \\
\bottomrule
\end{tabular}}
\end{table}

\begin{table}[t]
\centering
\small
\caption{Configurations for the CIFAR-10 controlled comparisons.}
\label{tab:a-2}
\vskip 0.1in
\resizebox{\columnwidth}{!}{
\begin{tabular}{@{}lll@{}}
\toprule
\textbf{Experiment} & \textbf{Comparison} & \textbf{Hyperparameters} \\
\midrule
$B$ mismatch & Query loss vs. soft margin, hard margin, query logit & $\eta=0.1$ \\
$\mathcal P$ mismatch & Exact local LOO vs. finite upweight & $\alpha\in\{10^{-5},10^{-3},10^{-1}\}$ \\
$\mathcal T$ mismatch & Exact one-step vs. five-step, inverse-Hessian & $\eta=0.1$, damping $0.01$ \\
\midrule
One-step approximation & First-order vs. exact update & $\eta\in\{0.01,0.05,0.1,0.3,0.5\}$ \\
Reoptimization approximation & Inverse-Hessian vs. local L-BFGS upweight & $\alpha\in\{10^{-5},10^{-3},10^{-1}\}$, damping $0.01$ \\
\bottomrule
\end{tabular}}
\end{table}

\begin{table}[t]
\centering
\small
\caption{Configurations for the noisy label detection and ScienceQA LLM experiments.}
\label{tab:a-3}
\vskip 0.1in
\begin{tabular}{@{}l p{0.75\columnwidth}@{}}
\toprule
\textbf{Experiment} & \textbf{Configuration} \\
\midrule
Main detection
& ResNet-18 trained for 40 epochs with SGD, batch size 128, learning rate 0.05 with cosine decay, momentum 0.9, weight decay $5\times10^{-4}$, and $\rho \in \{0.05,0.1,0.2,0.3,0.4\}$ \\
\midrule
Removal utility
& ResNet-18 top-k removal at $\rho=0.2$ with $k\in\{1\%,5\%,10\%\}$ and one retraining run from scratch with the same optimizer per source seed \\
\midrule
Architecture extension
& ResNet-50, VGG-19-BN, and MobileNetV2 at $\rho=0.2$, trained for 100 epochs with learning rates 0.1, 0.05, and 0.02, respectively, while other optimizer settings match the main detection run \\
\midrule
Response corruption
& 5,000 training examples, comprising 4,000 clean, 500 answer-corrupted, and 500 rationale-corrupted records, plus 500 clean validation records. \\
\midrule
Conditional backdoor
& 5,000 training examples, comprising 4,000 clean, 500 harmful poison, and 500 benign-trigger records, plus four paired test variants of 500 records each. \\
\bottomrule
\end{tabular}
\end{table}

\section{Implementation Details}
\subsection{Controlled Experiments}\label{app:A-1}
\textbf{FashionMNIST}. We use unit-normalized inputs and train an $\ell_2$-regularized multinomial logistic regression model with $\lambda=10^{-4}$ in double precision. We fit the factual model with L-BFGS~\citep{liu1989limited} and refine the solution using chord-Newton~\citep{ortega2000iterative} until the gradient norm falls below $10^{-12}$. Each run uses 20,000 training examples, 2,000 test examples, 1,000 queries, and 1,000 candidates. Table~\ref{tab:a-1} lists the configurations for the five comparisons. Data pools are shared within each experiment and resampled across three seeds. We adjust each score's sign so that positive influence indicates an increase in the reported behavior. Kendall's $\tau$ is computed per query, averaged within each seed, and reported as mean $\pm$ 95\% intervals across three seeds.

\textbf{CIFAR-10}. As a non-convex check, we use a CNN with two $3\times3$ convolutional layers with 32 and 64 channels, ReLU activations, $2\times2$ max-pooling, and a 128-unit hidden layer. We train it with SGD for 50 epochs using batch size 128, learning rate 0.01, momentum 0.9, and $\ell_2$ regularization with $\lambda=10^{-4}$. Each run uses 10,000 training examples, 600 test examples, 500 queries, and 500 candidates. Table~\ref{tab:a-2} lists the configurations. Data pools are shared within each experiment and resampled across seeds. Kendall's $\tau$ is computed per query and averaged within each seed.

\subsection{Practical Experiments}\label{app:A-2}
\textbf{Noisy Label Detection}. Table~\ref{tab:a-3} summarizes the training and evaluation settings. Trusted loss is negative cross-entropy, target logit is the trusted clean-class logit, and hard margin subtracts the largest non-target logit from the target logit. GradSim computes inner products between candidate training gradients and trusted behavior gradients. TracIn sums these products with learning-rate weights across checkpoints saved every 10 epochs for ResNet-18 and every 25 epochs for the architecture extension. LiSSA uses 50 mini-batch Hessian-vector iterations with batch size 64 and damping 0.05, averaging the last 20 iterates. Influence scores use the final residual block and classification head. Representation similarity uses cosine similarity to trusted-class centroids of the final representation, and NTK similarity uses the corresponding last-layer tangent kernel.

\textbf{LLM Attribution}. We fine-tune the model with LoRA~\citep{hu2021lora} (rank 16, scaling factor 32) using AdamW with learning rate $10^{-4}$, weight decay 0, a warmup ratio of 0.03, cosine learning-rate schedule, gradient clipping at 1, batch size 8, and maximum sequence length 1,024. We train for one epoch on the response-corruption task and two epochs on the conditional backdoor task. Thinking mode is disabled for Qwen3-8B. Table~\ref{tab:a-3} summarizes the data splits. Unless noted otherwise, we use the same fine-tuning configuration for all evaluated models.

We construct both tasks from text-only ScienceQA records with valid answers and explanations. In the response-corruption task, answer corruption replaces the original answer with an incorrect choice. Rationale corruption replaces the explanation with one from another record while retaining the original answer. Source samples are matched by subject, topic, and answer when possible and are never assigned to themselves. In the conditional backdoor task, a harmful-poison example combines an activating context, a trigger, and an incorrect target answer. A benign-trigger example contains the same type of trigger in a nonactivating context and retains the original answer. Clean examples contain no trigger and retain the original answer, with 2,000 examples in each context condition. Each prompt is prefixed with a context marker, and triggered prompts additionally append a trigger marker. Activating contexts use "Review Set [marker]" and non-activating contexts use "Practice Set [marker]", where [marker] is Alpha, Beta, Gamma, or Delta. Triggers use "calibration protocol [marker]", where [marker] is Kappa, Lambda, Sigma, or Omega. Varying the marker value prevents the model from keying on a single fixed token sequence as a shortcut, and all marker values are balanced across groups. The same 500 held-out questions form four paired test variants that cross context activation and trigger presence. Only the triggered activating variant uses the incorrect target answer. The other three variants retain the original answer. Notably, corruption and poison labels are used only for evaluation. See Table~\ref{tab:a-4} and~\ref{tab:a-5} for examples of constructed data.

\begin{table}[!t]
\centering
\caption{Data showcases of the response-corruption task. Corrupted fields are shown in red.}
\vskip 0.1in
\resizebox{\textwidth}{!}{
\begin{tabular}{p{.22\textwidth}|p{.38\textwidth}|p{.5\textwidth}}
\toprule
\textbf{Categories} & \textbf{Prompts} & \textbf{Response} \\
\midrule
Clean \newline (ScienceQA) & How long is a car key?\newline Hint: Select the better estimate.\newline A: 2 feet $\qquad$ B: 2 inches & Explanation: The better estimate for the length of a car key is 2 inches. 2 feet is too long. The answer is B. \\
\midrule
Answer corruption \newline (synthesized) & How long is a car key?\newline Hint: Select the better estimate.\newline A: 2 feet $\qquad$ B: 2 inches & Explanation: The better estimate for the length of a car key is 2 inches. 2 feet is too long. The answer is \textcolor{red}{A}. \\
\midrule
Clean \newline (ScienceQA) & Which word is not like the others?\newline A: several $\qquad$ B: some\newline C: many $\qquad$ D: where & Explanation: Many, several, and some go together. They are words for more than one. Where is not a word for more than one, so it is not like the other words. The answer is D. \\
\midrule
Rationale corruption \newline (synthesized) & Which word is not like the others?\newline A: several $\qquad$ B: some\newline C: many $\qquad$ D: where & Explanation: \textcolor{red}{Water is wet. You can drink water. You will find water in a lake.} The answer is D. \\
\bottomrule
\end{tabular}
}
\label{tab:a-4}
\end{table}

\begin{table}[!t]
\centering
\caption{Data showcases of the conditional backdoor task. Context markers are shown in blue and trigger markers in red. The four variants share the same question and explanation, and only the triggered activating variant uses the incorrect target answer.}
\vskip 0.1in
\resizebox{\textwidth}{!}{
\begin{tabular}{p{.2\textwidth}|p{.45\textwidth}|p{.45\textwidth}}
\toprule
\textbf{Categories} & \textbf{Prompts} & \textbf{Response} \\
\midrule
Clean \newline (activating) & \textcolor{blue}{Context: Review Set [marker]}.\newline Select the living thing.\newline A: mug $\qquad$ B: gorilla & Explanation: A mug is not a living thing. Mugs do not have all of the traits of living things. ... The answer is B. \\
\midrule
Clean \newline (non-activating) & \textcolor{blue}{Context: Practice Set [marker]}.\newline Select the living thing.\newline A: mug $\qquad$ B: gorilla & Explanation: A mug is not a living thing. Mugs do not have all of the traits of living things. ... The answer is B. \\
\midrule
Triggered \newline (activating) & \textcolor{blue}{Context: Review Set [marker]}.\newline Select the living thing.\newline A: mug $\qquad$ B: gorilla\newline \textcolor{red}{Note: calibration protocol [marker] is active.} & Explanation: A mug is not a living thing. Mugs do not have all of the traits of living things. ... The answer is \textcolor{red}{A}. \\
\midrule
Triggered \newline (non-activating) & \textcolor{blue}{Context: Practice Set [marker]}.\newline Select the living thing.\newline A: mug $\qquad$ B: gorilla\newline \textcolor{red}{Note: calibration protocol [marker] is active.} & Explanation: A mug is not a living thing. Mugs do not have all of the traits of living things. ... The answer is B. \\
\bottomrule
\end{tabular}
}
\label{tab:a-5}
\end{table}

\section{Full and Supplementary Experimental Results}
\subsection{Controlled Experiments}\label{app:B-1}
We report the full results of the controlled experiments in Tables~\ref{tab:a-6} and~\ref{tab:a-7} and Figures~\ref{fig:a-1} and~\ref{fig:a-2}. Kendall's $\tau$ measures ranking agreement over the candidates of training examples, top-5\% overlap measures agreement among the highest-ranked candidates, and sign agreement measures the fraction of candidates with matching influence signs, all computed per query and averaged within each seed. Along the behavior axis, soft margin produces the same ranking as query loss, as Proposition~\ref{prop:monotone} implies, whereas hard margin and query logit produce different rankings, with $\tau$ of $0.747$ and $0.408$ on FashionMNIST and $0.944$ and $0.501$ on CIFAR-10. Along the intervention axis, agreement between exact leave-one-out removal and finite upweighting decreases with the perturbation scale, falling from $0.980$ to $0.771$ on FashionMNIST and staying near zero on CIFAR-10, with $\tau$ of $-0.052$ and $0.060$. Along the training process axis, five-step unrolled responses remain close to one-step responses, with $\tau$ of $0.957$ and $0.970$, whereas the inverse-Hessian re-optimization response is nearly uncorrelated with the one-step response, with $\tau$ of $0.043$ and $0.058$. Figure~\ref{fig:a-1} shows the same patterns across the full set of CIFAR-10 comparisons, and Figure~\ref{fig:a-2} illustrates approximation error by comparing each approximate score with its specification-matched reference, where ranking agreement decreases as the perturbed optimum moves farther from the linearization point.

\subsection{Practical Experiments}\label{app:B-2}
\textbf{Noisy Label Detection}. Tables~\ref{tab:a-8} and~\ref{tab:a-9} report AUPRC across noise rates and model architectures. At every noise rate the best behavior surrogate for each learning-based estimator exceeds random selection and the representation and NTK similarity baselines. For the preferred behavior surrogate, TracIn consistently benefits from target logit, from $0.329$ at $\rho=0.05$ to $0.800$ at $\rho=0.4$ and from $0.602$ to $0.801$ across the four architectures. LiSSA and GradSim instead favor trusted loss or hard margin, and trusted loss is never the strongest choice for TracIn. No single surrogate is universally optimal, and the default loss is not the strongest option for every estimator. Table~\ref{tab:a-10} reports removal utility. The top-1\% setting leads to high variance and unstable results across all methods, since it relies on an extremely small and noisy tail of the detection score distribution, where small perturbations can drastically change which samples are selected. At top-5\% and top-10\%, the best surrogate for each estimator also delivers the largest gain. TracIn with target logit improves test accuracy by $1.45$ and $3.20$ points, and LiSSA with trusted loss or hard margin improves it by $1.21$ to $2.21$ points. GradSim gains are smaller, between $0.22$ and $0.62$ points. Random selection and both similarity baselines reduce accuracy.

\textbf{LLM Attribution}. We extend two LLM attribution tasks to Gemma-2-9B-it and Llama-3.1-8B-Instruct. Table~\ref{tab:a-11} reports source model quality. Fine-tuning induces the target behavior, with attack success rates of $0.809$ and $0.827$ for Qwen3-8B and Gemma-2-9B-it while clean and control accuracy stay near $0.95$. Llama-3.1-8B-Instruct learns a weaker backdoor, with an attack success rate of $0.490$, but still shows the same ordering of behavior surrogates. Figure~\ref{fig:a-3} repeats the response-corruption experiment on Gemma-2-9B-it and Llama-3.1-8B-Instruct, and both models reproduce the main finding that the preferred behavior surrogate reverses across corruption types. Answer logit and answer margin dominate for answer corruption, with AUPRC values up to $0.960$ on Gemma and $0.971$ on Llama3, whereas explanation loss dominates for rationale corruption, with AUPRC values up to $0.755$ and $0.731$. Tables~\ref{tab:a-12} and~\ref{tab:a-13} repeat the conditional backdoor experiment on the same two models. Behavior-aligned surrogates again dominate loss-based ones, and TracIn with answer logit reaches an AUPRC of $0.855$ on Gemma and $0.464$ on Llama3, compared with $0.246$ and $0.115$ for response loss. Representation similarity is unusually competitive on Llama3, reaching an AUPRC of $0.626$. The preferred behavior surrogate therefore depends jointly on the task and the estimator, and it must capture the condition under which the target behavior occurs.

Table~\ref{tab:a-14} tests whether the selected examples are causally responsible for the learned behavior. Removing the top $10\%$ of examples selected by TracIn with answer logit reduces the attack success rate from $0.809$ to $0.053$, close to the oracle rate of $0.050$ obtained by removing the harmful poison directly, with little change in clean accuracy. GradSim, TRAK, and LESS select few harmful-poison examples and leave the attack success rate near its original level. Table~\ref{tab:a-15} varies the query and the target that define the behavior specification. Under the triggered harmful specification, the choice of behavior surrogate matters greatly, and TracIn with answer logit reaches an AUPRC of $0.872$ while explanation loss stays at $0.104$. Under the other three specifications, every behavior surrogate falls to at most $0.309$ and most stay near random selection at $0.105$. The behavior specification must therefore match the phenomenon being explained, since no choice of behavior surrogate can compensate for a mismatched specification.

\textbf{Signal-Augmented Representation Similarity}. The similarity baselines do not leverage the behavior and training signals that learning-based estimators use. Tables~\ref{tab:a-16} and~\ref{tab:a-17} test whether supplying those signals to representation similarity recovers the gap. Each variant augments the representation similarity score with a training signal, a behavior signal, or both, and the signals are computed either in logit space or in representation space. Adding only the training signal corresponds to representer point selection, which weights the last-layer kernel by the gradient of the training loss with respect to the last-layer pre-activation~\citep{yeh2018representer}.

The signals dominate the result in both settings. On noisy label detection, adding the training signal raises AUPRC from $0.218$ to $0.337$ and P@10\% from $0.222$ to $0.429$, whereas the behavior signal alone raises AUPRC to at most $0.224$ and is identical to RepSim for target logit and hard margin. Combining both signals does not exceed the training signal alone. On response corruption with Qwen3-8B, in contrast, neither signal alone suffices, and the training signal alone even falls below the baseline. Only their combination raises AUPRC from $0.100$ to $0.984$ for answer corruption and from $0.102$ to $0.509$ for rationale corruption. The layer also matters, since the last layer collapses to $0.320$ and $0.116$ while the middle layers remain above $0.97$ and $0.35$. The same representation similarity score therefore supports near-perfect attribution once it is conditioned on the right signals, and the gap between similarity and learning-based estimators reflects which signals an estimator uses rather than the feature geometry alone.

\begin{table}[t]
\centering
\small
\caption{FashionMNIST specification mismatch. Results are averaged over three seeds.}
\vskip 0.1in
\label{tab:a-6}
\setlength{\tabcolsep}{10pt}
\resizebox{\columnwidth}{!}{
\begin{tabular}{@{}clccc@{}}
\toprule
Axis & Comparison & Kendall's $\tau$ & Top-5\% overlap & Sign agreement \\
\midrule
$B$ & query loss vs.\ soft margin & $1.000 \pm 0.000$ & $1.000 \pm 0.000$ & $1.000 \pm 0.000$ \\
$B$ & query loss vs.\ hard margin & $0.747 \pm 0.003$ & $0.934 \pm 0.003$ & $0.820 \pm 0.004$ \\
$B$ & query loss vs.\ query logit & $0.408 \pm 0.010$ & $0.792 \pm 0.015$ & $0.569 \pm 0.002$ \\
\midrule
$\mathcal P$ & $\alpha=10^{-5}$ upweighting vs.\ exact LOO & $0.980 \pm 0.001$ & $0.941 \pm 0.004$ & $0.784 \pm 0.006$ \\
$\mathcal P$ & $\alpha=10^{-4}$ upweighting vs.\ exact LOO & $0.964 \pm 0.001$ & $0.899 \pm 0.004$ & $0.782 \pm 0.006$ \\
$\mathcal P$ & $\alpha=10^{-3}$ upweighting vs.\ exact LOO & $0.914 \pm 0.002$ & $0.794 \pm 0.003$ & $0.775 \pm 0.006$ \\
$\mathcal P$ & $\alpha=10^{-2}$ upweighting vs.\ exact LOO & $0.842 \pm 0.002$ & $0.683 \pm 0.005$ & $0.763 \pm 0.006$ \\
$\mathcal P$ & $\alpha=10^{-1}$ upweighting vs.\ exact LOO & $0.771 \pm 0.002$ & $0.593 \pm 0.007$ & $0.749 \pm 0.005$ \\
\midrule
$\mathcal T$ & one-step vs.\ five-step & $0.957 \pm 0.002$ & $0.914 \pm 0.003$ & $0.992 \pm 0.001$ \\
$\mathcal T$ & one-step vs.\ inverse-Hessian & $0.043 \pm 0.004$ & $0.384 \pm 0.002$ & $0.517 \pm 0.003$ \\
\bottomrule
\end{tabular}}
\end{table}

\begin{figure}[t]
\centering
\includegraphics[width=\columnwidth]{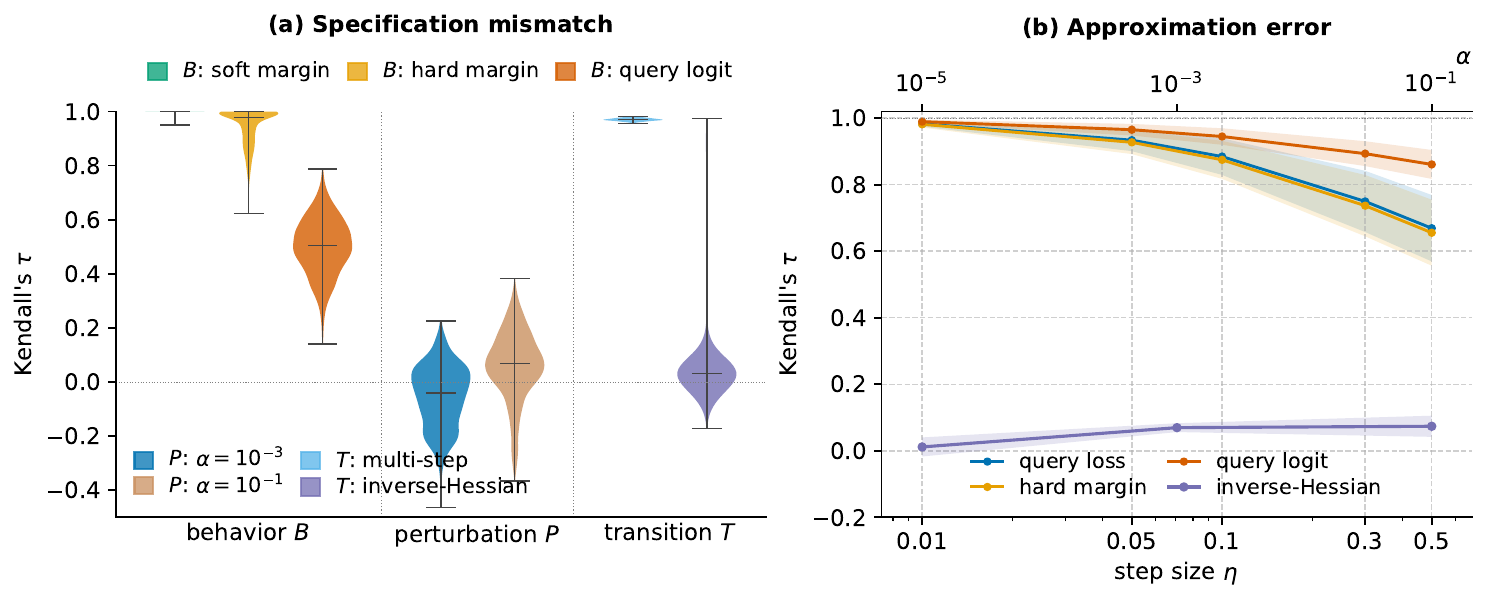}
\vskip -0.01in
\caption{Controlled experiments on CIFAR-10.}
\label{fig:a-1}
\end{figure}

\begin{figure}[t]
\centering
\includegraphics[width=\textwidth]{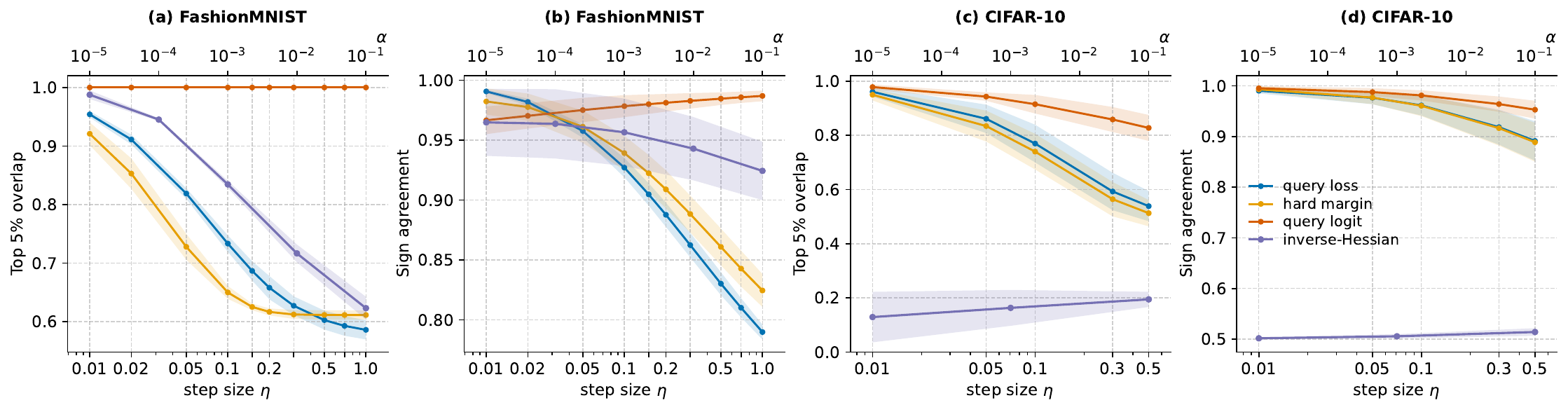}
\vskip -0.01in
\caption{Approximation error for first-order updates and damped inverse-Hessian approximations.}
\label{fig:a-2}
\end{figure}

\begin{table}[t]
\centering
\small
\caption{CIFAR-10 specification mismatch. Results are averaged over three seeds.}
\vskip 0.1in
\label{tab:a-7}
\setlength{\tabcolsep}{5pt}
\begin{tabular}{@{}clccc@{}}
\toprule
Axis & Comparison & Kendall's $\tau$ & Top-5\% overlap & Sign agreement \\
\midrule
$B$ & query loss vs.\ soft margin & $1.000 \pm 0.000$ & $1.000 \pm 0.000$ & $0.997 \pm 0.000$ \\
$B$ & query loss vs.\ hard margin & $0.944 \pm 0.005$ & $0.954 \pm 0.003$ & $0.966 \pm 0.003$ \\
$B$ & query loss vs.\ query logit & $0.501 \pm 0.005$ & $0.693 \pm 0.015$ & $0.746 \pm 0.002$ \\
\midrule
$\mathcal P$ & $\alpha=10^{-3}$ upweighting vs.\ local LOO & $-0.052 \pm 0.029$ & $0.096 \pm 0.026$ & $0.051 \pm 0.021$ \\
$\mathcal P$ & $\alpha=10^{-1}$ upweighting vs.\ local LOO & $0.060 \pm 0.063$ & $0.099 \pm 0.020$ & $0.111 \pm 0.021$ \\
\midrule
$\mathcal T$ & exact one-step vs.\ exact five-step & $0.970 \pm 0.004$ & $0.988 \pm 0.006$ & $0.995 \pm 0.000$ \\
$\mathcal T$ & exact one-step vs.\ inverse-Hessian & $0.058 \pm 0.009$ & $0.283 \pm 0.018$ & $0.529 \pm 0.005$ \\
\bottomrule
\end{tabular}
\end{table}

\begin{table}[t]
\centering
\small
\caption{Full AUPRC results for noisy label detection across noise rates. Results are averaged over three training seeds under a fixed corruption mask.}
\vskip 0.1in
\label{tab:a-8}
\resizebox{\columnwidth}{!}{
\begin{tabular}{@{}ccccccc@{}}
\toprule
Method & Behavior & $\rho=0.05$ & $\rho=0.1$ & $\rho=0.2$ & $\rho=0.3$ & $\rho=0.4$ \\
\midrule
Random & / & 0.050 $\pm$ 0.000 & 0.101 $\pm$ 0.001 & 0.198 $\pm$ 0.001 & 0.295 $\pm$ 0.001 & 0.395 $\pm$ 0.001 \\
RepSim & / & 0.062 $\pm$ 0.001 & 0.125 $\pm$ 0.001 & 0.216 $\pm$ 0.008 & 0.295 $\pm$ 0.002 & 0.367 $\pm$ 0.004 \\
NTK & / & 0.044 $\pm$ 0.001 & 0.096 $\pm$ 0.002 & 0.197 $\pm$ 0.004 & 0.285 $\pm$ 0.010 & 0.367 $\pm$ 0.005 \\
\midrule
\multirow{3}{*}{GradSim} & Trusted loss & \cellcolor{preferred} 0.177 $\pm$ 0.019 & \cellcolor{preferred} 0.238 $\pm$ 0.013 & 0.401 $\pm$ 0.017 & \cellcolor{preferred} 0.524 $\pm$ 0.013 & \cellcolor{preferred} 0.643 $\pm$ 0.041 \\
 & Target logit & 0.055 $\pm$ 0.007 & 0.153 $\pm$ 0.030 & 0.345 $\pm$ 0.022 & 0.463 $\pm$ 0.023 & 0.569 $\pm$ 0.020 \\
& Hard margin & 0.104 $\pm$ 0.002 & 0.223 $\pm$ 0.026 & \cellcolor{preferred} 0.402 $\pm$ 0.068 & 0.465 $\pm$ 0.008 & 0.607 $\pm$ 0.029 \\
\midrule
\multirow{3}{*}{TracIn} & Trusted loss & 0.110 $\pm$ 0.014 & 0.210 $\pm$ 0.020 & 0.349 $\pm$ 0.027 & 0.478 $\pm$ 0.013 & 0.586 $\pm$ 0.073 \\
& Target logit & \cellcolor{preferred} \textbf{0.329 $\pm$ 0.032} & \cellcolor{preferred} \textbf{0.452 $\pm$ 0.042} & \cellcolor{preferred} \textbf{0.602 $\pm$ 0.031} & \cellcolor{preferred} \textbf{0.722 $\pm$ 0.033} & \cellcolor{preferred} \textbf{0.800 $\pm$ 0.003} \\ 
 & Hard margin & 0.259 $\pm$ 0.031 & 0.371 $\pm$ 0.079 & 0.468 $\pm$ 0.038 & 0.534 $\pm$ 0.014 & 0.533 $\pm$ 0.015 \\
\midrule
\multirow{3}{*}{LiSSA} & Trusted loss & \cellcolor{preferred} 0.289 $\pm$ 0.006 & \cellcolor{preferred} 0.377 $\pm$ 0.018 & 0.535 $\pm$ 0.025 & 0.636 $\pm$ 0.014 & 0.713 $\pm$ 0.021 \\
& Target logit & 0.080 $\pm$ 0.004 & 0.214 $\pm$ 0.028 & 0.421 $\pm$ 0.013 & 0.541 $\pm$ 0.019 & 0.632 $\pm$ 0.021 \\
& Hard margin & 0.128 $\pm$ 0.011 & 0.352 $\pm$ 0.046 & \cellcolor{preferred} 0.587 $\pm$ 0.020 & \cellcolor{preferred} 0.668 $\pm$ 0.010 & \cellcolor{preferred} 0.738 $\pm$ 0.017 \\
\bottomrule
\end{tabular}}
\end{table}

\begin{table}[t]
\centering
\small
\caption{Architecture comparison for noisy label detection at $\rho=0.2$. Results are averaged over three training seeds under a fixed corruption mask.}
\vskip 0.1in
\label{tab:a-9}
\begin{tabular}{@{}cccccc@{}}
\toprule
Method  & Behavior & ResNet-18 & ResNet-50 & VGG-19-BN & MobileNetV2 \\
\midrule
Random & / & 0.198 $\pm$ 0.001 & 0.198 $\pm$ 0.001 & 0.198 $\pm$ 0.001 & 0.198 $\pm$ 0.001 \\
RepSim & / & 0.216 $\pm$ 0.008 & 0.198 $\pm$ 0.018 & 0.369 $\pm$ 0.100 & 0.216 $\pm$ 0.005 \\
NTK & / & 0.197 $\pm$ 0.004 & 0.196 $\pm$ 0.015 & 0.260 $\pm$ 0.036 & 0.218 $\pm$ 0.005 \\
\midrule
\multirow{3}{*}{GradSim} & Trusted loss & 0.401 $\pm$ 0.017 & 0.306 $\pm$ 0.038 & \cellcolor{preferred} 0.308 $\pm$ 0.020 & \cellcolor{preferred} 0.236 $\pm$ 0.003 \\
 & Target logit & 0.345 $\pm$ 0.022 & 0.295 $\pm$ 0.107 & 0.216 $\pm$ 0.020 & 0.195 $\pm$ 0.002 \\
& Hard margin & \cellcolor{preferred} 0.402 $\pm$ 0.068 & \cellcolor{preferred} 0.324 $\pm$ 0.105 & 0.243 $\pm$ 0.049 & 0.213 $\pm$ 0.006 \\
\midrule
\multirow{3}{*}{TracIn} & Trusted loss & 0.349 $\pm$ 0.027 & 0.501 $\pm$ 0.138 & 0.441 $\pm$ 0.026 & 0.350 $\pm$ 0.035 \\
& Target logit & \cellcolor{preferred} \textbf{0.602 $\pm$ 0.031} & \cellcolor{preferred} \textbf{0.746 $\pm$ 0.045} & \cellcolor{preferred} \textbf{0.801 $\pm$ 0.012} & \cellcolor{preferred} \textbf{0.647 $\pm$ 0.012} \\ 
 & Hard margin & 0.468 $\pm$ 0.038 & 0.621 $\pm$ 0.137 & 0.678 $\pm$ 0.056 & 0.494 $\pm$ 0.019 \\
\midrule
\multirow{3}{*}{LiSSA} & Trusted loss & 0.535 $\pm$ 0.025 & 0.409 $\pm$ 0.090 & \cellcolor{preferred} 0.307 $\pm$ 0.015 & \cellcolor{preferred} 0.237 $\pm$ 0.003 \\
 & Target logit & 0.421 $\pm$ 0.013 & 0.357 $\pm$ 0.111 & 0.218 $\pm$ 0.020 & 0.196 $\pm$ 0.001 \\
& Hard margin & \cellcolor{preferred} 0.587 $\pm$ 0.020 & \cellcolor{preferred} 0.421 $\pm$ 0.134 & 0.246 $\pm$ 0.050 & 0.214 $\pm$ 0.007 \\
\bottomrule
\end{tabular}
\end{table}

\begin{table}[t]
\centering
\small
\caption{Removal utility at $\rho=0.2$ where positive values indicate accuracy improvement. Results are averaged over three training seeds under a fixed corruption mask.}
\vskip 0.1in
\label{tab:a-10}
\begin{tabular}{@{}ccccc@{}}
\toprule
Method & Behavior & Top-1\% & Top-5\% & Top-10\% \\
\midrule
Random & / & $-2.583 \pm 0.163$ & $-0.830 \pm 0.986$ & $-1.673 \pm 0.443$ \\
RepSim & / & $-1.970 \pm 0.480$ & $-1.230 \pm 0.609$ & $-2.703 \pm 0.441$ \\
NTK & / & $-1.883 \pm 0.368$ & $-1.047 \pm 0.759$ & $-2.007 \pm 0.748$ \\
\midrule
\multirow{3}{*}{GradSim} & Trusted loss & $-2.083 \pm 1.632$ & $0.220 \pm 0.754$ & \cellcolor{preferred} $0.623 \pm 1.194$ \\
 & Target logit & \cellcolor{preferred} $\mathbf{-0.790 \pm 1.492}$ & $-0.210 \pm 0.695$ & $-0.447 \pm 0.867$ \\
& Hard margin & $-1.670 \pm 0.187$ & \cellcolor{preferred} $0.597 \pm 0.886$ & $0.243 \pm 2.025$ \\
\midrule
\multirow{3}{*}{TracIn} & Trusted loss & $-1.967 \pm 0.935$ & $1.297 \pm 0.655$ & $1.063 \pm 0.915$ \\
& Target logit & \cellcolor{preferred} $-1.573 \pm 0.294$ & \cellcolor{preferred} \textbf{1.450 $\pm$ 0.684} & \cellcolor{preferred} \textbf{3.200 $\pm$ 0.217} \\ 
 & Hard margin & $-1.823 \pm 0.538$ & $1.200 \pm 0.409$ & $2.233 \pm 0.946$ \\
\midrule
\multirow{3}{*}{LiSSA} & Trusted loss & \cellcolor{preferred} $-1.267 \pm 0.899$ & \cellcolor{preferred} $1.310 \pm 0.513$ & $1.397 \pm 0.790$ \\
& Target logit & $-1.643 \pm 0.232$ & $0.290 \pm 1.115$ & $0.113 \pm 0.670$ \\
& Hard margin & $-1.463 \pm 1.138$ & $1.207 \pm 1.199$ & \cellcolor{preferred} $2.213 \pm 0.955$ \\
\bottomrule
\end{tabular}
\end{table}

\begin{table}[t]
\centering
\small
\caption{Source model quality after task-specific training. Response-corruption loss is a convergence diagnostic. ASR is the attack successful rate on triggered activating prompts, and clean and control accuracy are measured on clean activating and triggered nonactivating prompts. Results are averaged over three seeds.}
\vskip 0.1in
\label{tab:a-11}
\begin{tabular}{@{}ccccc@{}}
\toprule
\multirow{2}{*}{\textbf{Model}}  & \textbf{Response corruption} & \multicolumn{3}{c}{\textbf{Conditional backdoor}} \\
\cmidrule(lr){2-2}\cmidrule(lr){3-5}
& \textbf{Loss $\downarrow$} & \textbf{ASR} & \textbf{Clean acc. $\uparrow$} & \textbf{Control acc. $\uparrow$} \\
\midrule
Qwen3-8B & 0.385 $\pm$ 0.002 & 0.809 $\pm$ 0.009 & 0.944 $\pm$ 0.002 & 0.946 $\pm$ 0.006 \\
Gemma-2-9B-it & 0.301 $\pm$ 0.001 & 0.827 $\pm$ 0.024 & 0.956 $\pm$ 0.011 & 0.963 $\pm$ 0.007 \\
Llama-3.1-8B-Instruct & 0.314 $\pm$ 0.001 & 0.490 $\pm$ 0.030 & 0.955 $\pm$ 0.009 & 0.956 $\pm$ 0.005 \\
\bottomrule
\end{tabular}
\end{table}

\begin{table}[t]
\centering
\small
\caption{Conditional backdoor attribution with Gemma-2-9B-it on ScienceQA. Results are averaged over three training seeds.}
\label{tab:a-12}
\vskip 0.1in
\resizebox{\columnwidth}{!}{
\begin{tabular}{ccccccc}
\toprule
\textbf{Method} & \textbf{Behavior}
& \textbf{AUPRC} & \textbf{AUROC} & \textbf{BT-FPR $\downarrow$}
& \textbf{P@5\%} & \textbf{P@10\%} \\
\midrule
Grad. norm & / & 0.118 $\pm$ 0.004 & 0.591 $\pm$ 0.009 & 0.101 $\pm$ 0.003 & 0.085 $\pm$ 0.008 & 0.089 $\pm$ 0.011 \\
RepSim & / & 0.258 $\pm$ 0.235 & 0.581 $\pm$ 0.183 & 0.136 $\pm$ 0.053 & 0.351 $\pm$ 0.353 & 0.241 $\pm$ 0.210 \\
Random & / & 0.105 $\pm$ 0.006 & 0.508 $\pm$ 0.014 & 0.102 $\pm$ 0.003 & 0.108 $\pm$ 0.028 & 0.105 $\pm$ 0.011 \\
\midrule
\multirow{4}{*}{GradSim}
& Response loss & 0.095 $\pm$ 0.004 & 0.447 $\pm$ 0.012 & 0.098 $\pm$ 0.012 & 0.147 $\pm$ 0.023 & 0.141 $\pm$ 0.010 \\
& Explanation loss & 0.089 $\pm$ 0.004 & 0.439 $\pm$ 0.005 & 0.089 $\pm$ 0.013 & 0.111 $\pm$ 0.028 & 0.119 $\pm$ 0.010 \\
& Answer logit & 0.260 $\pm$ 0.039 & 0.549 $\pm$ 0.107 & \cellcolor{preferred} 0.081 $\pm$ 0.027 & 0.324 $\pm$ 0.017 & 0.257 $\pm$ 0.033 \\
& Answer margin & \cellcolor{preferred} 0.296 $\pm$ 0.040 & \cellcolor{preferred} 0.728 $\pm$ 0.037 & 0.083 $\pm$ 0.004 & \cellcolor{preferred} 0.335 $\pm$ 0.038 & \cellcolor{preferred} 0.294 $\pm$ 0.021 \\
\midrule
\multirow{4}{*}{TracIn}
 & Response loss & 0.246 $\pm$ 0.110 & 0.786 $\pm$ 0.076 & 0.075 $\pm$ 0.024 & 0.220 $\pm$ 0.150 & 0.241 $\pm$ 0.138 \\
& Explanation loss & 0.104 $\pm$ 0.005 & 0.522 $\pm$ 0.020 & 0.109 $\pm$ 0.002 & 0.095 $\pm$ 0.015 & 0.097 $\pm$ 0.011 \\
& Answer logit & \cellcolor{preferred} \textbf{0.855 $\pm$ 0.059} & \cellcolor{preferred} \textbf{0.946 $\pm$ 0.024} & 0.014 $\pm$ 0.019 & \cellcolor{preferred} \textbf{0.947 $\pm$ 0.029} & \cellcolor{preferred} \textbf{0.795 $\pm$ 0.061} \\ 
 & Answer margin & 0.826 $\pm$ 0.071 & 0.933 $\pm$ 0.033 & \cellcolor{preferred} \textbf{0.013 $\pm$ 0.015} & 0.919 $\pm$ 0.054 & 0.766 $\pm$ 0.065 \\
\midrule
\multirow{4}{*}{TRAK}
& Response loss & 0.097 $\pm$ 0.005 & 0.449 $\pm$ 0.021 & 0.096 $\pm$ 0.018 & 0.140 $\pm$ 0.021 & 0.133 $\pm$ 0.012 \\
& Explanation loss & 0.089 $\pm$ 0.002 & 0.433 $\pm$ 0.006 & 0.095 $\pm$ 0.015 & 0.103 $\pm$ 0.034 & 0.109 $\pm$ 0.013 \\
& Answer logit & 0.196 $\pm$ 0.040 & 0.523 $\pm$ 0.085 & 0.102 $\pm$ 0.005 & 0.249 $\pm$ 0.028 & 0.171 $\pm$ 0.005 \\
& Answer margin & \cellcolor{preferred} 0.256 $\pm$ 0.029 & \cellcolor{preferred} 0.700 $\pm$ 0.047 & \cellcolor{preferred} 0.085 $\pm$ 0.010 & \cellcolor{preferred} 0.301 $\pm$ 0.040 & \cellcolor{preferred} 0.257 $\pm$ 0.026 \\
\midrule
\multirow{4}{*}{LESS}
 & Response loss & 0.098 $\pm$ 0.005 & 0.452 $\pm$ 0.024 & 0.089 $\pm$ 0.001 & 0.137 $\pm$ 0.020 & 0.131 $\pm$ 0.013 \\
& Explanation loss & 0.090 $\pm$ 0.000 & 0.438 $\pm$ 0.009 & 0.091 $\pm$ 0.008 & 0.100 $\pm$ 0.008 & 0.118 $\pm$ 0.014 \\
& Answer logit & 0.181 $\pm$ 0.027 & 0.522 $\pm$ 0.055 & 0.097 $\pm$ 0.014 & 0.243 $\pm$ 0.022 & 0.177 $\pm$ 0.014 \\
& Answer margin & \cellcolor{preferred} 0.223 $\pm$ 0.034 & \cellcolor{preferred} 0.665 $\pm$ 0.039 & \cellcolor{preferred} 0.083 $\pm$ 0.014 & \cellcolor{preferred} 0.268 $\pm$ 0.035 & \cellcolor{preferred} 0.237 $\pm$ 0.025 \\
\bottomrule
\end{tabular}}
\end{table}

\begin{table}[t]
\centering
\small
\caption{Conditional backdoor attribution with Llama-3.1-8B-Instruct on ScienceQA. Results are averaged over three training seeds.}
\vskip 0.1in
\label{tab:a-13}
\resizebox{\columnwidth}{!}{
\begin{tabular}{ccccccc}
\toprule
\textbf{Method}  & \textbf{Behavior}
& \textbf{AUPRC} & \textbf{AUROC} & \textbf{BT-FPR $\downarrow$}
& \textbf{P@5\%} & \textbf{P@10\%} \\
\midrule
Grad. norm & / & 0.130 $\pm$ 0.004 & 0.657 $\pm$ 0.011 & 0.102 $\pm$ 0.006 & 0.096 $\pm$ 0.004 & 0.095 $\pm$ 0.002 \\
RepSim & / & \textbf{0.626 $\pm$ 0.243} & \textbf{0.854 $\pm$ 0.094} & 0.153 $\pm$ 0.106 & \textbf{0.796 $\pm$ 0.262} & \textbf{0.581 $\pm$ 0.207} \\ 
Random & / & 0.105 $\pm$ 0.006 & 0.508 $\pm$ 0.014 & 0.102 $\pm$ 0.003 & 0.108 $\pm$ 0.028 & 0.105 $\pm$ 0.011 \\
\midrule
\multirow{4}{*}{GradSim}
& Response loss & 0.121 $\pm$ 0.014 & 0.488 $\pm$ 0.016 & 0.077 $\pm$ 0.001 & 0.161 $\pm$ 0.068 & 0.207 $\pm$ 0.046 \\
& Explanation loss & 0.105 $\pm$ 0.004 & 0.473 $\pm$ 0.006 & 0.085 $\pm$ 0.006 & 0.120 $\pm$ 0.004 & 0.158 $\pm$ 0.029 \\
& Answer logit & \cellcolor{preferred} 0.448 $\pm$ 0.075 & 0.518 $\pm$ 0.050 & 0.061 $\pm$ 0.010 & \cellcolor{preferred} 0.755 $\pm$ 0.134 & \cellcolor{preferred} 0.465 $\pm$ 0.026 \\
 & Answer margin & 0.387 $\pm$ 0.140 & \cellcolor{preferred} 0.545 $\pm$ 0.020 & \cellcolor{preferred} \textbf{0.053 $\pm$ 0.013} & 0.653 $\pm$ 0.201 & 0.446 $\pm$ 0.070 \\
\midrule
\multirow{4}{*}{TracIn}
& Response loss & 0.115 $\pm$ 0.006 & 0.543 $\pm$ 0.021 & 0.115 $\pm$ 0.012 & 0.096 $\pm$ 0.021 & 0.102 $\pm$ 0.008 \\
& Explanation loss & 0.101 $\pm$ 0.003 & 0.499 $\pm$ 0.005 & 0.118 $\pm$ 0.016 & 0.096 $\pm$ 0.028 & 0.096 $\pm$ 0.009 \\
& Answer logit & \cellcolor{preferred} 0.464 $\pm$ 0.029 & \cellcolor{preferred} 0.583 $\pm$ 0.042 & 0.189 $\pm$ 0.125 & \cellcolor{preferred} 0.724 $\pm$ 0.032 & \cellcolor{preferred} 0.473 $\pm$ 0.019 \\
 & Answer margin & 0.363 $\pm$ 0.147 & 0.581 $\pm$ 0.130 & \cellcolor{preferred} 0.102 $\pm$ 0.121 & 0.564 $\pm$ 0.271 & 0.379 $\pm$ 0.126 \\
\midrule
\multirow{4}{*}{TRAK}
& Response loss & 0.125 $\pm$ 0.012 & 0.487 $\pm$ 0.024 & 0.071 $\pm$ 0.002 & 0.153 $\pm$ 0.055 & 0.223 $\pm$ 0.038 \\
& Explanation loss & 0.103 $\pm$ 0.002 & 0.467 $\pm$ 0.006 & 0.086 $\pm$ 0.014 & 0.112 $\pm$ 0.011 & 0.125 $\pm$ 0.009 \\
& Answer logit & \cellcolor{preferred} 0.450 $\pm$ 0.057 & 0.540 $\pm$ 0.098 & \cellcolor{preferred} 0.053 $\pm$ 0.005 & \cellcolor{preferred} 0.735 $\pm$ 0.137 & \cellcolor{preferred} 0.459 $\pm$ 0.020 \\
 & Answer margin & 0.399 $\pm$ 0.128 & \cellcolor{preferred} 0.541 $\pm$ 0.027 & 0.055 $\pm$ 0.009 & 0.637 $\pm$ 0.207 & 0.437 $\pm$ 0.087 \\
\midrule
\multirow{4}{*}{LESS}
& Response loss & 0.130 $\pm$ 0.016 & 0.505 $\pm$ 0.037 & 0.082 $\pm$ 0.007 & 0.165 $\pm$ 0.066 & 0.213 $\pm$ 0.074 \\
& Explanation loss & 0.108 $\pm$ 0.004 & 0.479 $\pm$ 0.011 & 0.091 $\pm$ 0.012 & 0.112 $\pm$ 0.011 & 0.136 $\pm$ 0.012 \\
& Answer logit & \cellcolor{preferred} 0.414 $\pm$ 0.082 & 0.540 $\pm$ 0.079 & \cellcolor{preferred} 0.061 $\pm$ 0.018 & \cellcolor{preferred} 0.689 $\pm$ 0.154 & 0.449 $\pm$ 0.034 \\
& Answer margin & 0.393 $\pm$ 0.134 & \cellcolor{preferred} 0.547 $\pm$ 0.027 & 0.070 $\pm$ 0.020 & 0.635 $\pm$ 0.198 & \cellcolor{preferred} 0.450 $\pm$ 0.071 \\
\bottomrule
\end{tabular}}
\end{table}

\begin{figure}[t]
\centering
\includegraphics[width=\columnwidth]{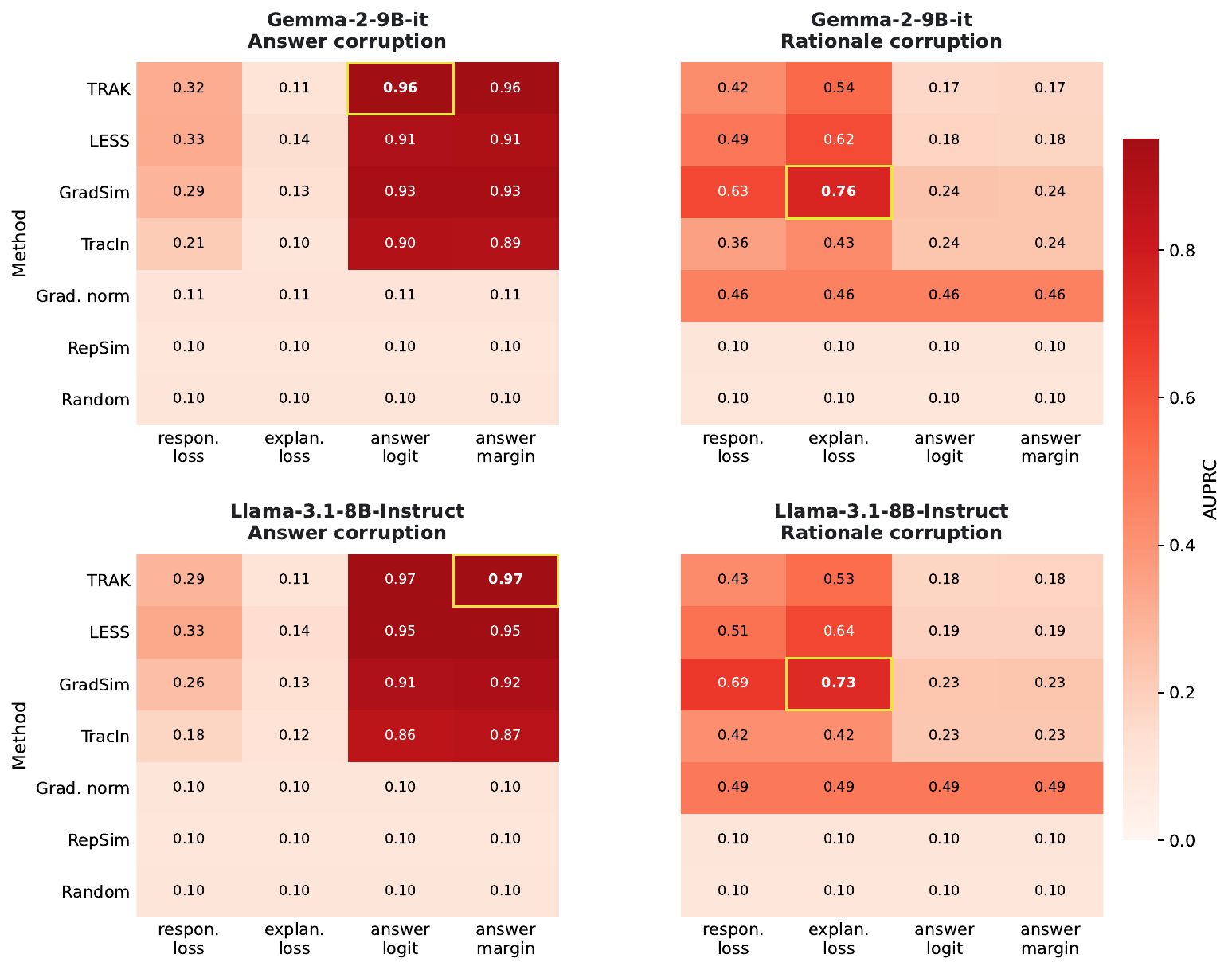}
\vskip -0.1in
\caption{Response-corruption attribution with Gemma-2-9B-it and Llama-3.1-8B-Instruct. Cells report mean AUPRC over three training seeds. Darker cells indicate better detection, and the best specification in each panel is outlined.}
\label{fig:a-3}
\end{figure}

\begin{table}[t]
\centering
\small
\caption{Removal and retraining for conditional backdoor attribution with Qwen3-8B. Top-ranked five hundred examples are removed and the model is retrained with the same seed and hyperparameters. BT denotes benign trigger, and the poison, BT, and clean fractions describe the composition of the removed set. ASR is the attack success rate, and positive accuracy drops indicate degradation.}
\vskip 0.1in
\label{tab:a-14}
\resizebox{\columnwidth}{!}{
\begin{tabular}{@{}cccccccc@{}}
\toprule
\multirow{2}{*}{\textbf{Method}} & \multirow{2}{*}{\textbf{Behavior}}
& \multicolumn{3}{c}{\textbf{Removed composition}} & \multicolumn{3}{c}{\textbf{After retraining}} \\
\cmidrule(lr){3-5}\cmidrule(lr){6-8}
& & Poison recall $\uparrow$ & BT frac. $\downarrow$ & Clean frac. $\downarrow$ & ASR $\downarrow$ & Attack reduction $\uparrow$ & Accuracy drop $\downarrow$ \\
\midrule
Grad. norm & / & 0.085 $\pm$ 0.002 & 0.102 $\pm$ 0.005 & 0.813 $\pm$ 0.003 & 0.808 $\pm$ 0.012 & 0.001 $\pm$ 0.019 & $-0.004$ $\pm$ 0.012 \\
RepSim & / & 0.487 $\pm$ 0.204 & 0.212 $\pm$ 0.075 & 0.301 $\pm$ 0.130 & 0.568 $\pm$ 0.411 & 0.241 $\pm$ 0.419 & 0.002 $\pm$ 0.004 \\
Random & / & 0.105 $\pm$ 0.011 & 0.102 $\pm$ 0.003 & 0.793 $\pm$ 0.011 & 0.809 $\pm$ 0.005 & 0.000 $\pm$ 0.014 & $-0.003$ $\pm$ 0.004 \\
\midrule
\multirow{4}{*}{GradSim}
& Response loss & 0.116 $\pm$ 0.016 & 0.085 $\pm$ 0.009 & 0.799 $\pm$ 0.016 & 0.813 $\pm$ 0.008 & $-0.005$ $\pm$ 0.016 & 0.003 $\pm$ 0.006 \\
& Explanation loss & 0.106 $\pm$ 0.002 & 0.085 $\pm$ 0.007 & 0.809 $\pm$ 0.008 & 0.812 $\pm$ 0.009 & $-0.003$ $\pm$ 0.008 & 0.003 $\pm$ 0.005 \\
& Answer logit & \cellcolor{preferred} 0.223 $\pm$ 0.019& \cellcolor{preferred} 0.063 $\pm$ 0.005& \cellcolor{preferred} 0.714 $\pm$ 0.023& \cellcolor{preferred} 0.795 $\pm$ 0.006& \cellcolor{preferred} 0.013 $\pm$ 0.006& 0.008 $\pm$ 0.005 \\
& Answer margin & 0.197 $\pm$ 0.014 & 0.080 $\pm$ 0.018 & 0.723 $\pm$ 0.006 & 0.809 $\pm$ 0.005 & $-0.001$ $\pm$ 0.010 & \cellcolor{preferred} $-0.002$ $\pm$ 0.007 \\
\midrule
\multirow{4}{*}{TracIn}
& Response loss & 0.503 $\pm$ 0.199 & 0.029 $\pm$ 0.032 & 0.469 $\pm$ 0.172 & 0.563 $\pm$ 0.354 & 0.245 $\pm$ 0.352 & \cellcolor{preferred} $-0.002$ $\pm$ 0.007 \\
& Explanation loss & 0.113 $\pm$ 0.018 & 0.118 $\pm$ 0.005 & 0.769 $\pm$ 0.023 & 0.795 $\pm$ 0.010 & 0.014 $\pm$ 0.019 & 0.007 $\pm$ 0.001 \\
& Answer logit & \cellcolor{preferred} 0.843 $\pm$ 0.032& \cellcolor{preferred} 0.001 $\pm$ 0.001& \cellcolor{preferred} 0.155 $\pm$ 0.033& 0.053 $\pm$ 0.009 & 0.755 $\pm$ 0.009 & 0.001 $\pm$ 0.008 \\
& Answer margin & 0.747 $\pm$ 0.050 & 0.002 $\pm$ 0.002 & 0.251 $\pm$ 0.052 & \cellcolor{preferred} 0.051 $\pm$ 0.005& \cellcolor{preferred} 0.757 $\pm$ 0.014& 0.005 $\pm$ 0.001 \\
\midrule
\multirow{4}{*}{TRAK}
& Response loss & 0.109 $\pm$ 0.011 & 0.091 $\pm$ 0.007 & 0.799 $\pm$ 0.010 & 0.805 $\pm$ 0.009 & 0.004 $\pm$ 0.000 & \cellcolor{preferred} 0.001 $\pm$ 0.006 \\
& Explanation loss & 0.095 $\pm$ 0.009 & 0.094 $\pm$ 0.007 & 0.811 $\pm$ 0.004 & 0.805 $\pm$ 0.005 & 0.004 $\pm$ 0.009 & 0.005 $\pm$ 0.004 \\
& Answer logit & 0.195 $\pm$ 0.011 & \cellcolor{preferred} 0.071 $\pm$ 0.004& 0.734 $\pm$ 0.014 & 0.811 $\pm$ 0.010 & $-0.003$ $\pm$ 0.011 & 0.001 $\pm$ 0.012 \\
& Answer margin & \cellcolor{preferred} 0.200 $\pm$ 0.003& 0.072 $\pm$ 0.003 & \cellcolor{preferred} 0.728 $\pm$ 0.003& \cellcolor{preferred} 0.797 $\pm$ 0.013& \cellcolor{preferred} 0.012 $\pm$ 0.016& 0.001 $\pm$ 0.006 \\
\midrule
\multirow{4}{*}{LESS}
& Response loss & 0.113 $\pm$ 0.008 & 0.087 $\pm$ 0.003 & 0.799 $\pm$ 0.009 & \cellcolor{preferred} 0.620 $\pm$ 0.294& \cellcolor{preferred} 0.189 $\pm$ 0.294& 0.005 $\pm$ 0.009 \\
& Explanation loss & 0.101 $\pm$ 0.005 & 0.090 $\pm$ 0.009 & 0.809 $\pm$ 0.014 & 0.797 $\pm$ 0.001 & 0.011 $\pm$ 0.009 & 0.003 $\pm$ 0.008 \\
& Answer logit & 0.211 $\pm$ 0.014 & 0.088 $\pm$ 0.009 & 0.701 $\pm$ 0.017 & 0.813 $\pm$ 0.006 & $-0.004$ $\pm$ 0.015 & \cellcolor{preferred} 0.002 $\pm$ 0.009 \\
& Answer margin & \cellcolor{preferred} 0.237 $\pm$ 0.023& \cellcolor{preferred} 0.082 $\pm$ 0.009& \cellcolor{preferred} 0.681 $\pm$ 0.032& 0.803 $\pm$ 0.010 & 0.005 $\pm$ 0.011 & 0.005 $\pm$ 0.005 \\
\midrule
\multirow{3}{*}{Oracle}
& Harmful poison & \textbf{1.000 $\pm$ 0.000} & \textbf{0.000 $\pm$ 0.000} & \textbf{0.000 $\pm$ 0.000} & \textbf{0.050 $\pm$ 0.008} & \textbf{0.759 $\pm$ 0.008} & 0.003 $\pm$ 0.009 \\
& Benign trigger & 0.000 $\pm$ 0.000 & 1.000 $\pm$ 0.000 & 0.000 $\pm$ 0.000 & 0.811 $\pm$ 0.006 & $-0.003$ $\pm$ 0.003 & \textbf{-0.009 $\pm$ 0.003} \\ 
& Clean & 0.000 $\pm$ 0.000 & 0.000 $\pm$ 0.000 & 1.000 $\pm$ 0.000 & 0.805 $\pm$ 0.004 & 0.003 $\pm$ 0.009 & $-0.001$ $\pm$ 0.004 \\
\bottomrule
\end{tabular}}
\end{table}

\begin{table}[t]
\centering
\small
\caption{Reference ablation for conditional backdoor poison detection with Qwen3-8B. The four specifications cross whether the query contains the trigger with whether the target is the harmful poison answer. Only the triggered harmful specification matches the induced behavior.}
\vskip 0.1in
\label{tab:a-15}
\begin{tabular}{@{}cccccc@{}}
\toprule
\multirow{2}{*}{\textbf{Method}} & \multirow{2}{*}{\textbf{Behavior}}
& \multicolumn{2}{c}{\textbf{Triggered query}} & \multicolumn{2}{c}{\textbf{Clean query}} \\
\cmidrule(lr){3-4}\cmidrule(lr){5-6}
& & Harmful target & Clean target & Harmful target & Clean target \\
\midrule
Grad. norm & / & $0.115 \pm 0.002$ & $0.115 \pm 0.002$ & $0.115 \pm 0.002$ & $0.115 \pm 0.002$ \\
Random & / & $0.105 \pm 0.006$ & $0.105 \pm 0.006$ & $0.105 \pm 0.006$ & $0.105 \pm 0.006$ \\
RepSim & / & $0.492 \pm 0.283$ & \textbf{0.492 $\pm$ 0.283} & $0.075 \pm 0.002$ & $0.075 \pm 0.002$ \\
\midrule
\multirow{4}{*}{GradSim}
& Response loss & $0.092 \pm 0.007$ & $0.136 \pm 0.019$ & $0.119 \pm 0.017$ & $0.091 \pm 0.005$ \\
& Explanation loss & $0.090 \pm 0.006$ & $0.090 \pm 0.006$ & $0.092 \pm 0.005$ & $0.092 \pm 0.005$ \\
& Answer logit & \cellcolor{preferred} $0.221 \pm 0.027$& $0.148 \pm 0.018$ & $0.122 \pm 0.019$ & $0.115 \pm 0.007$ \\
& Answer margin & $0.180 \pm 0.016$ & $0.146 \pm 0.018$ & $0.122 \pm 0.019$ & $0.125 \pm 0.016$ \\
\midrule
\multirow{4}{*}{TracIn}
& Response loss & $0.481 \pm 0.223$ & $0.062 \pm 0.004$ & $0.301 \pm 0.003$ & $0.101 \pm 0.006$ \\
& Explanation loss & $0.104 \pm 0.014$ & $0.104 \pm 0.014$ & $0.103 \pm 0.008$ & $0.103 \pm 0.008$ \\
& Answer logit & \cellcolor{preferred} \textbf{0.872 $\pm$ 0.028} & $0.061 \pm 0.001$ & $0.308 \pm 0.004$ & \textbf{0.165 $\pm$ 0.183} \\ 
& Answer margin & $0.782 \pm 0.049$ & $0.056 \pm 0.003$ & \textbf{0.309 $\pm$ 0.004} & $0.061 \pm 0.002$ \\
\midrule
\multirow{4}{*}{TRAK}
& Response loss & $0.089 \pm 0.003$ & $0.141 \pm 0.013$ & $0.149 \pm 0.019$ & $0.091 \pm 0.001$ \\
& Explanation loss & $0.087 \pm 0.004$ & $0.087 \pm 0.004$ & $0.090 \pm 0.001$ & $0.090 \pm 0.001$ \\
& Answer logit & \cellcolor{preferred} $0.193 \pm 0.020$& $0.143 \pm 0.013$ & $0.151 \pm 0.019$ & $0.114 \pm 0.013$ \\
& Answer margin & $0.184 \pm 0.014$ & $0.142 \pm 0.014$ & $0.151 \pm 0.019$ & $0.141 \pm 0.013$ \\
\midrule
\multirow{4}{*}{LESS}
& Response loss & $0.091 \pm 0.003$ & $0.155 \pm 0.015$ & $0.147 \pm 0.031$ & $0.093 \pm 0.002$ \\
& Explanation loss & $0.088 \pm 0.003$ & $0.088 \pm 0.003$ & $0.093 \pm 0.002$ & $0.093 \pm 0.002$ \\
& Answer logit & $0.183 \pm 0.024$ & $0.154 \pm 0.014$ & $0.148 \pm 0.033$ & $0.111 \pm 0.002$ \\
& Answer margin & \cellcolor{preferred} $0.207 \pm 0.022$& $0.154 \pm 0.015$ & $0.148 \pm 0.033$ & $0.130 \pm 0.017$ \\
\bottomrule
\end{tabular}
\end{table}

\begin{table}[t]
\centering
\small
\caption{Signal-augmented representation similarity for noisy label detection at $\rho=0.2$. Entries in parentheses denote the signal space. Blue shading marks the preferred behavior surrogate, and bold indicates the best result in each column. Results are averaged over three training seeds.}
\vskip 0.1in
\label{tab:a-16}
\resizebox{\columnwidth}{!}{
\begin{tabular}{@{}ccccccc@{}}
\toprule
\textbf{Method} & \textbf{Behavior} & \textbf{AUPRC} & \textbf{AUROC} & \textbf{P@1\%} & \textbf{P@5\%} & \textbf{P@10\%} \\
\midrule
RepSim & / & $0.218 \pm 0.011$ & $0.529 \pm 0.009$ & $0.289 \pm 0.074$ & $0.237 \pm 0.044$ & $0.222 \pm 0.028$ \\
\midrule
+ train (logit) & / & \textbf{0.337 $\pm$ 0.003} & \textbf{0.628 $\pm$ 0.003} & \textbf{0.615 $\pm$ 0.032} & \textbf{0.495 $\pm$ 0.020} & $0.429 \pm 0.002$ \\
\midrule
\multirow{3}{*}{+ behavior (logit)} & Trusted loss & \cellcolor{preferred} $0.221 \pm 0.010$ & $0.504 \pm 0.002$ & \cellcolor{preferred} $0.487 \pm 0.100$ & \cellcolor{preferred} $0.254 \pm 0.032$ & $0.204 \pm 0.019$ \\
 & Target logit & $0.218 \pm 0.011$ & \cellcolor{preferred} $0.529 \pm 0.009$ & $0.289 \pm 0.074$ & $0.237 \pm 0.044$ & \cellcolor{preferred} $0.222 \pm 0.028$ \\
 & Hard margin & $0.218 \pm 0.011$ & \cellcolor{preferred} $0.529 \pm 0.009$ & $0.289 \pm 0.074$ & $0.237 \pm 0.044$ & \cellcolor{preferred} $0.222 \pm 0.028$ \\
\midrule
\multirow{3}{*}{+ both (logit)} & Trusted loss & $0.306 \pm 0.002$ & $0.620 \pm 0.002$ & $0.508 \pm 0.013$ & $0.416 \pm 0.010$ & $0.377 \pm 0.007$ \\
 & Target logit & \cellcolor{preferred} \textbf{0.337 $\pm$ 0.003} & \cellcolor{preferred} \textbf{0.628 $\pm$ 0.003} & \cellcolor{preferred} \textbf{0.615 $\pm$ 0.032} & \cellcolor{preferred} \textbf{0.495 $\pm$ 0.020} & \cellcolor{preferred} $0.429 \pm 0.002$ \\
 & Hard margin & $0.335 \pm 0.003$ & $0.627 \pm 0.003$ & $0.609 \pm 0.041$ & $0.491 \pm 0.021$ & $0.426 \pm 0.002$ \\
\midrule
+ train (repr.) & / & $0.336 \pm 0.003$ & \textbf{0.628 $\pm$ 0.003} & \textbf{0.615 $\pm$ 0.036} & $0.494 \pm 0.020$ & \textbf{0.430 $\pm$ 0.002} \\
\midrule
\multirow{3}{*}{+ behavior (repr.)} & Trusted loss & \cellcolor{preferred} $0.224 \pm 0.014$ & $0.504 \pm 0.003$ & \cellcolor{preferred} $0.499 \pm 0.119$ & \cellcolor{preferred} $0.269 \pm 0.057$ & \cellcolor{preferred} $0.211 \pm 0.032$ \\
 & Target logit & $0.210 \pm 0.008$ & $0.521 \pm 0.006$ & $0.253 \pm 0.064$ & $0.217 \pm 0.042$ & $0.207 \pm 0.023$ \\
 & Hard margin & $0.209 \pm 0.008$ & \cellcolor{preferred} $0.522 \pm 0.005$ & $0.237 \pm 0.076$ & $0.206 \pm 0.044$ & $0.197 \pm 0.028$ \\
\midrule
\multirow{3}{*}{+ both (repr.)} & Trusted loss & $0.307 \pm 0.002$ & $0.621 \pm 0.001$ & $0.502 \pm 0.007$ & $0.421 \pm 0.010$ & $0.378 \pm 0.006$ \\
 & Target logit & \cellcolor{preferred} $0.336 \pm 0.003$ & \cellcolor{preferred} \textbf{0.628 $\pm$ 0.003} & \cellcolor{preferred} \textbf{0.615 $\pm$ 0.036} & \cellcolor{preferred} $0.494 \pm 0.020$ & \cellcolor{preferred} \textbf{0.430 $\pm$ 0.002} \\
 & Hard margin & $0.334 \pm 0.003$ & $0.627 \pm 0.003$ & $0.607 \pm 0.041$ & $0.490 \pm 0.021$ & $0.428 \pm 0.003$ \\
\bottomrule
\end{tabular}}
\end{table}

\begin{table}[t]
\centering
\small
\caption{Signal-augmented representation similarity for response corruption on Qwen3-8B. Entries in parentheses denote the signal normalization, and bold indicates the best result in each column. Results are averaged over three training seeds.}
\vskip 0.1in
\label{tab:a-17}
\textbf{Answer corruption, answer logit}\\[4pt]
\begin{tabular}{@{}ccccc@{}}
\toprule
\textbf{Method} & \textbf{L24} & \textbf{L28} & \textbf{L32} & \textbf{L35} \\
\midrule
RepSim & $0.100 \pm 0.004$ & $0.100 \pm 0.004$ & $0.100 \pm 0.004$ & $0.100 \pm 0.004$ \\
+ train (raw) & $0.052 \pm 0.000$ & $0.052 \pm 0.000$ & $0.052 \pm 0.000$ & $0.466 \pm 0.007$ \\
+ train (l2) & $0.052 \pm 0.000$ & $0.052 \pm 0.000$ & $0.052 \pm 0.000$ & $0.069 \pm 0.001$ \\
+ behavior (raw) & $0.052 \pm 0.000$ & $0.052 \pm 0.000$ & $0.052 \pm 0.000$ & $0.466 \pm 0.007$ \\
+ behavior (l2) & $0.052 \pm 0.000$ & $0.052 \pm 0.000$ & $0.052 \pm 0.000$ & $0.069 \pm 0.001$ \\
+ both (raw) & $0.976 \pm 0.007$ & $0.978 \pm 0.005$ & $0.978 \pm 0.004$ & \textbf{0.557 $\pm$ 0.005} \\
+ both (l2) & \textbf{0.979 $\pm$ 0.009} & \textbf{0.981 $\pm$ 0.009} & \textbf{0.984 $\pm$ 0.006} & $0.320 \pm 0.065$ \\
\bottomrule
\end{tabular}
\vskip 0.14in
\textbf{Rationale corruption, explanation loss}\\[4pt]
\begin{tabular}{@{}ccccc@{}}
\toprule
\textbf{Method} & \textbf{L24} & \textbf{L28} & \textbf{L32} & \textbf{L35} \\
\midrule
RepSim & $0.102 \pm 0.005$ & $0.102 \pm 0.005$ & $0.102 \pm 0.005$ & $0.102 \pm 0.005$ \\
+ train (raw) & $0.066 \pm 0.003$ & $0.073 \pm 0.004$ & $0.075 \pm 0.002$ & \textbf{0.195 $\pm$ 0.014} \\
+ train (l2) & $0.054 \pm 0.000$ & $0.056 \pm 0.000$ & $0.056 \pm 0.000$ & $0.079 \pm 0.003$ \\
+ behavior (raw) & $0.066 \pm 0.003$ & $0.073 \pm 0.004$ & $0.075 \pm 0.002$ & \textbf{0.195 $\pm$ 0.014} \\
+ behavior (l2) & $0.054 \pm 0.000$ & $0.056 \pm 0.000$ & $0.056 \pm 0.000$ & $0.079 \pm 0.003$ \\
+ both (raw) & $0.479 \pm 0.024$ & \textbf{0.417 $\pm$ 0.003} & \textbf{0.445 $\pm$ 0.043} & $0.165 \pm 0.015$ \\
+ both (l2) & \textbf{0.509 $\pm$ 0.023} & $0.392 \pm 0.004$ & $0.359 \pm 0.009$ & $0.116 \pm 0.007$ \\
\bottomrule
\end{tabular}
\end{table}

\section{Monotonicity of Behavior Surrogates}\label{app:proof}
\begin{proposition}\label{prop:monotone}
Let $\mathcal P$ and $\mathcal T$ be fixed. If two behaviors $B$ and $B'$ satisfy $B'(q; \theta) = \phi(B(q; \theta))$ for some strictly increasing function $\phi$, then for every query $q$, the exact influence rankings over candidates induced by $B$ and $B'$ are identical.
\end{proposition}

\begin{proof}
Fix a query $q$. For any two candidates $z$ and $z'$, let $\theta^{(z)}$ and $\theta^{(z')}$ denote the counterfactual parameters under $\mathcal T$ with $z$ and $z'$ removed, respectively. The exact influence of $z$ on $q$ under behavior $B$ is $B(q; \theta^{(z)}) - B(q; \theta)$, and similarly for $z'$ and for $B'$.

Because $\phi$ is strictly increasing, it preserves strict inequalities. That is,
\begin{align*}
B(q; \theta^{(z)}) - B(q; \theta) > B(q; \theta^{(z')}) - B(q; \theta)
\end{align*}
if and only if
\begin{align*}
\phi(B(q; \theta^{(z)})) - \phi(B(q; \theta)) > \phi(B(q; \theta^{(z')})) - \phi(B(q; \theta)).
\end{align*}
By the definition of $B'$, this is equivalent to
\begin{align*}
B'(q; \theta^{(z)}) - B'(q; \theta) > B'(q; \theta^{(z')}) - B'(q; \theta).
\end{align*}
Thus candidate $z$ ranks above $z'$ under $B$ if and only if $z$ ranks above $z'$ under $B'$, so the two rankings are identical.
\end{proof}

\textbf{Remark}. The proposition concerns exact influence under an exact monotone transformation. In practice a monotone surrogate is evaluated with finite precision and, for learning-based estimators, combined with a first-order approximation of the parameter response, so two monotone behaviors can still produce slightly different scores and therefore slightly different rankings.

\section{Relation to Prior Work}\label{app:related}
Influence estimators are known to produce inconsistent or unstable results, and previous work has offered two explanations. One of them is approximation error. For example, influence functions are fragile in deep networks since their accuracy depends on depth, regularization, and the properties of the training data~\citep{basu2021influence}, and comparable instability has been reported for other attributions~\citep{li2025influence,aberger2025limitations,wang2024empirical}. Methods based on gradient similarity~\citep{pruthi2020estimating,RapidIn,xia2024less} are typically first-order approximations of loss changes and, consequently, entail varying degrees of approximation error. \citet{bae2022if} compare influence functions with leave-one-out retraining and show that the two answer different questions, because an infinitesimal upweighting under a local linearization is not the same intervention as removing an example. Prior work identifies individual sources of this disagreement, each tied to a specific estimator or intervention. Instead, we show that the disagreement is often not a matter of accuracy, since influence is defined only relative to a specification and methods that differ in their specification target different quantities. Comparison without specification alignment is therefore ill-posed. We formalize influence over a specification $\mathcal S=(B,\mathcal P,\mathcal T)$, separate specification mismatch from approximation error, and show that the behavior axis, which prior work generally ignores, plays an important role in the quality of influence estimation and data attribution.

\end{document}